\documentclass[11pt]{article}

\usepackage[preprint]{acl}

\usepackage{times}
\usepackage{latexsym}

\usepackage{algorithm}
\usepackage{algpseudocode}
\usepackage{bm}
\usepackage{amssymb}
\usepackage{amsmath}
\usepackage{booktabs}
\usepackage{tikz}
\usetikzlibrary{positioning}
\usepackage{amsthm}
\newtheorem{proposition}{Proposition}

\usepackage[most]{tcolorbox}
\usepackage{fvextra}

\newtcolorbox{promptbox}[2][]{%
  enhanced,
  colback=white, colframe=black,
  colbacktitle=black, coltitle=white,
  title={#2}, halign title=center,
  fonttitle=\ttfamily\small,
  fontupper=\ttfamily\small,
  halign=left,
  arc=2mm, boxrule=1pt,
  left=6pt, right=6pt, top=4pt, bottom=4pt,
  before upper={\setlength{\parindent}{0pt}},
  #1
}

\DefineVerbatimEnvironment{PromptText}{Verbatim}{%
  breaklines, breaksymbolleft={}, breakautoindent=false,
  breakindent=0pt, fontsize=\small}

\usepackage[T1]{fontenc}

\usepackage[utf8]{inputenc}

\usepackage{microtype}

\usepackage{inconsolata}

\usepackage{graphicx}

\usepackage{subcaption}

\title{Silicon Samples Match the Answers but Not the Respondents}

\author{Yusuke Kumagae \\
  Hakuhodo DY Holdings Inc.\\
  \texttt{yusuke.kumagae@hakuhodo.co.jp} \\\And
  Masahiro Kaneko    \\
  MBZUAI\\
  \texttt{Masahiro.Kaneko@mbzuai.ac.ae} \\}

\begin{document}
\maketitle
\begin{abstract}
Silicon sampling generates survey respondents by having a large language model (LLM) answer a survey as a person with given demographic attributes, and the resulting synthetic populations are increasingly used in place of human ones.
Existing evaluations focus on individual-answer accuracy, per-question answer distributions, and correlations across questions.
None of these compares the respondents themselves, so a generator whose respondents have combinations of answers found in no human respondent can still score well if it reproduces every answer distribution.
We propose the respondent coverage gap (RCG)\footnote{Our code is available from the project page: \url{https://respondent-coverage.github.io/}.}, which
represents each respondent as a point defined by their complete answer profile
and measures how far each human respondent is from the nearest generated one, relative to the distance between human samples.
On two US social surveys, all six LLMs we evaluate leave human respondents uncovered, 
and under our main settings, they fall behind a baseline that samples each answer
independently from the human marginals.
Across the prompting and inference settings we test, every LLM has a higher RCG than a Gaussian copula fitted to the human data.
\end{abstract}
\section{Introduction}\label{sec:intro}

\begin{figure}[t]
  \centering
  \includegraphics[width=\columnwidth]{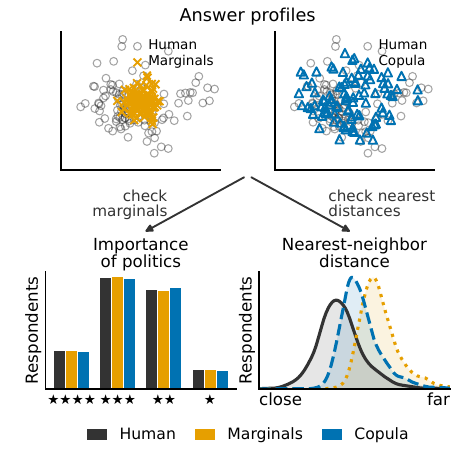}
  \caption{Marginals reproduces each question's answer distribution, and Copula additionally models
  pairwise dependence between questions. Top: answer profiles of human and generated respondents. Bottom: existing evaluation compares answers per question, where the three groups are identical (left); ours compares respondents, where humans are farther from the nearest generated respondent than from the nearest other human (right).}
  \label{fig:teaser}
\end{figure}

Large language models (LLMs) can act as a specified person and answer questions from that person's standpoint~\citep{Zhang2018-rj,Shanahan2023-zq,Wang2024-ki}.
\emph{Silicon sampling}~\citep{Argyle2023-kp} exploits this ability: it gives an LLM a combination of demographic attributes (e.g., female, in her thirties) as a persona, has it produce a full set of answers to a questionnaire (an \emph{answer profile}) as a human respondent would, and repeats this over many personas to obtain a synthetic population.
Such synthetic populations are increasingly used as substitutes for, or supplements to,
human populations in opinion polls, market research, and social science experiments~\citep{Santurkar2023-vx,Park2024-ii,Horton2023-yp,Aher2023-zc}.
Typical uses include estimating the distribution of opinions, probing the reactions of groups that are hard to recruit, and piloting a survey before fielding it.

The quality of such a synthetic population has so far been evaluated through the distribution of answers to each question.
\emph{Individual accuracy}~\citep{Park2024-ii,Hu2026-wb} counts, question by question, whether the answer the LLM generates from a respondent's attributes matches that respondent's own answer.
\emph{Marginal fidelity}~\citep{Santurkar2023-vx,Durmus2024-hc} compares, for each question, the distribution of human answers with the distribution of answers generated from the same attributes.
\emph{Correlational fidelity}~\citep{Moon2024-ql,Williams2026-bs,Libovicky2026-ta,Ma2026-my} computes the correlation between answers to different questions separately on the human side and on the generated side, and compares the two.
All of them aggregate answers along the axis of questions; none of them compares the \emph{respondent distribution}, that is, how respondents are spread out as answer profiles.

Agreement in answer distributions does not guarantee agreement in respondent distributions.
Even when the answer distributions for every question match,
the synthetic population can consist of generated respondents whose answer profiles are absent from the human population;
for a synthetic population to stand in for a human one, the respondent distributions must agree.
We fit two generators to the World Values Survey (WVS)~\citep{wvs7}: \emph{Marginals} reproduces only the per-question answer distributions, and \emph{Copula} additionally models pairwise dependence between questions.
Per question, both are indistinguishable from the human respondents (\autoref{fig:teaser}, bottom left).
Per respondent, measuring the distance from each human respondent to the nearest generated one separates the two (\autoref{fig:teaser}, bottom right; details in Appendix~\ref{sec:appendix:teaser}): Marginals lies far from the human respondents, and Copula, despite modeling pairwise dependence, still falls short of the human reference.
Per-question distributions and pairwise correlations do not fully characterize the respondent distribution.

In this paper, we propose the \emph{respondent coverage gap} (RCG), a metric of how well a synthetic population covers the respondent distribution of a human population.
We treat each answer profile as a point and measure, from each human point, the distance to the nearest generated point;
this directly quantifies how close each human respondent is to a generated respondent.
Because this distance depends on the number of questions and the size of the populations, we normalize it by the same distance measured among the human points themselves.
The reference is another human sample rather than the human respondents themselves: a synthetic population should not copy the human respondents but resemble them as a population.
A ratio of one means the synthetic population is as close to the human respondents as another human sample would be; the more it exceeds one, the farther the least covered human respondents are from any generated respondent.

Evaluating six LLMs (two proprietary models from OpenAI and Google and four open-weight models from 9B to 70B parameters) on two real-world surveys of US residents,
we find that none of them cover the human respondents.
Under our main settings, all fall short of Marginals, which reproduces nothing but
the per-question answer distributions.
Some settings bring individual models close to or below Marginals: using the full response history, increasing reasoning effort, or using higher-tier models such as gpt-5.4 and gemini-3.1-pro-preview.
However, no tested setting yields a lower RCG than Copula,
which additionally models pairwise dependence between questions.
Calibrating the generated answer distributions to match the human ones,
question by question, reduces RCG below that of Marginals, but not below that of Copula.

This failure of LLMs is invisible to existing metrics.
Marginal fidelity does not separate generators that differ in respondent coverage,
and an aggregate correlation metric assigns Marginals, which samples every question independently,
a score on par with the split-half ceiling of the human data.
The failure also takes model-specific forms: some models contract onto typical profiles, others drift off the support of the human population.
These failures surface only when the respondent distribution is examined; by that standard,
synthetic populations from the evaluated LLMs are not yet adequate substitutes for human ones.

\section{Related Work}\label{sec:related_work}
\subsection{Silicon Sampling and Its Evaluation}\label{sec:related_work_silicon_sampling}
Silicon sampling has been applied across domains such as
social simulations~\citep{Park2023-ng,Park2024-ii},
economic experiments~\citep{Horton2023-yp},
social science experiments~\citep{Ashokkumar2026-lf,Aher2023-zc},
and social surveys~\citep{Santurkar2023-vx,Argyle2023-kp,Ma2025-qr}.

\paragraph{Individual accuracy.}
This line of work scores, question by question, whether the generated answer matches
the answer observed from the corresponding human respondent.
\citet{Argyle2023-kp} address the task of predicting vote choice and measure the correspondence with tetrachoric correlations.
\citet{Park2024-ii} ground LLM agents in interviews with Americans and predict
their held-out answers to the General Social Survey and the Big Five Inventory.
Yet the scoring remains per-question. \citet{Park2024-ii} themselves note that
agents can predict a respondent's answers well yet still misrepresent how attitudes covary across the
sample, and defer sample-level correlational structure to future work.

\paragraph{Marginal fidelity.}
In contrast, this line of work compares answer distributions question by question,
without pairing generated and human respondents.
\citet{Santurkar2023-vx} measure the 1-Wasserstein distance between model and human answer distributions over 1,498 Pew survey questions.
\citet{Durmus2024-hc} apply the same per-question comparison across 60+ countries, scoring one minus the Jensen--Shannon distance.
In all of these, the human reference is itself an aggregate distribution; no individual enters the comparison.

\subsection{Beyond Marginal Evaluation}\label{sec:related_work_beyond}
Both lines of work treat each question in isolation.
A few studies~\citep{Moon2024-ql,Williams2026-bs,Libovicky2026-ta,Ma2026-my} take a step further and evaluate dependence across questions.
To do so, they share a common two-step procedure:
(1) computing inter-item associations separately within the human sample and within
the generated sample, and
(2) comparing the two association matrices as wholes.
While these methods do capture dependence across questions, the comparison is made at the level of question pairs, not of the individual respondents that silicon sampling promises to substitute for.
\citet{Williams2026-bs} themselves defer individual-level structure to future work.

\subsection{Comparing Generated and Human Respondents as Distributions}\label{sec:compare_distributions}
Outside silicon sampling, generated samples are compared with human ones as whole distributions,
each sample being one point.
\citet{Sajjadi2018-dv} decompose this comparison into precision,
how closely the generated samples resemble the human ones (fidelity),
and recall, how well they span the variability of the human ones (diversity).
\citet{Kynkaanniemi2019-we} and \citet{Naeem2020-sg} estimate both with $k$-nearest-neighbor balls,
and MAUVE \citep{Pillutla2021-pg} measures both for text as a divergence between the two distributions.
These methods use discrete membership information (whether a sample lies inside a neighborhood ball or which quantization cell it occupies) rather than retaining its distance to the human samples.
Prior work also compares synthetic-to-human nearest-neighbor distances with human-to-human distances.
\citet{Alaa2022-uk} use such a comparison in a binary authenticity test for each generated sample,
whereas \citet{Lautrup2025-xj} use a ratio of median distances in med.DCR to assess privacy risk.
Our focus is instead on human respondents whose nearest generated respondent is too far away.

Silicon sampling has so far been evaluated along the axis of questions
(\autoref{sec:related_work_silicon_sampling} and \autoref{sec:related_work_beyond});
we instead compare distributions at the respondent level, treating each respondent’s
complete answer profile as one sample.
\section{Evaluation Framework}\label{sec:framework}
Let $q_1, \dots, q_M$ be $M$ questions.
A respondent is represented as an answer vector $\bm{u} = (u_1, \dots, u_M)$,
where $u_m$ is the answer chosen for question $q_m$.
We write $\mathcal{H}$ for a set of human answer vectors and
$\mathcal{G}$ for a set of generated answer vectors.

\subsection{Conditional Generation}\label{sec:conditional_generation}
We generate each respondent by conditioning an LLM on a set of demographic
attributes (e.g., sex and age), and compare the generated respondents $\mathcal{G}$
against the human respondents $\mathcal{H}$ who share the same attributes.
The LLM answers the questions one at a time, and the prompt for each question includes a
few of that respondent's earlier answers.
This sequential format, also employed in prior work~\citep{Moon2024-ql,Libovicky2026-ta},\footnote{
\citet{Moon2024-ql} and \citet{Libovicky2026-ta} feed the full history of earlier answers;
we sample only a few of them, since with the full history the total prompt length over a respondent's
$M$ questions grows as $O(M^2)$.
\autoref{sec:ablation} and Appendix~\ref{sec:appendix:full_history} examine the full history.\label{fn:history_size}}
allows later answers to depend on earlier ones within the same respondent.\footnote{
Most prior work instead queries each question independently~\citep{Santurkar2023-vx,Durmus2024-hc,AlKhamissi2024-tj,Dominguez-Olmedo2024-vn,Williams2026-bs,Taday-Morocho2026-yg};
among generated respondents who share the same demographic attributes, this pairs the answers
to different questions at random, so that no answer depends on any other.}

\subsection{Respondent Coverage Gap}\label{sec:respondent_coverage}
Given a distance $d(\bm{u}, \bm{v})$ between answer vectors, let $\delta(\bm{u}; \mathcal{S})$ be the
distance from an answer vector $\bm{u}$ to its nearest neighbor in a set $\mathcal{S}$.\footnote{Excluding $\bm{u}$ itself when $\bm{u} \in \mathcal{S}$.}
Our core idea is to examine the distribution of the nearest-neighbor distance
$\delta(\bm{u}; \mathcal{G})$ over the human respondents $\bm{u} \in \mathcal{H}$ and,
symmetrically, of $\delta(\bm{v}; \mathcal{H})$ over the generated respondents $\bm{v} \in \mathcal{G}$.
The former asks, for each human respondent, how close the nearest generated respondent is;
the latter asks, for each generated respondent, how close the nearest human respondent is.

Taking the maximum of the former recovers the directed Hausdorff distance~\citep{hausdorff1914grundzuge}:
\begin{equation}
d_{\mathcal{H}\rightarrow\mathcal{G}} = \max_{\bm{u} \in \mathcal{H}} \delta(\bm{u}; \mathcal{G}).
\end{equation}
We modify this distance in two ways.
First, we replace the maximum with the 95th percentile, denoted $d^{95}_{\mathcal{H}\rightarrow\mathcal{G}}$,
for robustness to outliers~\citep{Huttenlocher1993-th}.
Second, we normalize it by the same quantity computed within the human respondents,
$d^{95}_{\mathcal{H}\rightarrow\mathcal{H}}$,
the 95th percentile of the distances from each human respondent to their nearest other human respondent;
this makes the ratio comparable across different choices of $\mathcal{H}$.

We define the \emph{respondent coverage gap} (RCG) and its reverse-direction companion,
the \emph{off-support ratio} (OSR), which share this human-side quantity as the denominator:
\begin{align}
\mathrm{RCG}(\mathcal{H}, \mathcal{G}) &=
\frac{d^{95}_{\mathcal{H}\rightarrow\mathcal{G}}}{d^{95}_{\mathcal{H}\rightarrow\mathcal{H}}}, \label{eq:rc}\\
\mathrm{OSR}(\mathcal{H}, \mathcal{G}) &=
\frac{d^{95}_{\mathcal{G}\rightarrow\mathcal{H}}}{d^{95}_{\mathcal{H}\rightarrow\mathcal{H}}}. \label{eq:osr}
\end{align}
Since $\delta(\bm{u}; \mathcal{S})$ shrinks as $\mathcal{S}$ grows,
we use equally sized $\mathcal{H}$ and $\mathcal{G}$ in our main evaluation (\autoref{sec:cells} and \autoref{sec:appendix:cell}).
RCG measures the extent to which human respondents are left uncovered by the generated ones;
OSR reverses the direction of the numerator while keeping the same denominator, and
thus measures the extent to which generated respondents lie outside the support of the human respondents.
We use RCG as our main metric and OSR as a diagnostic when analyzing failure modes (\autoref{sec:failure_modes}).

\begin{figure}[t]
  \centering
  \includegraphics[width=\columnwidth]{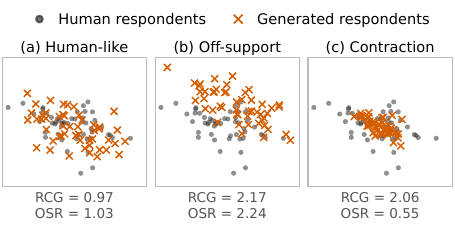}
  \caption{
    Toy examples of the human-like, off-support, and contraction patterns
    on two-dimensional synthetic data (50 human and 50 generated respondents,
    Euclidean distance).
    The panels illustrate representative configurations, not an exhaustive
    classification of RCG and OSR values.
  }
  \label{fig:rcg_osr_toy}
\end{figure}

Both ratios are expected to be near one when $\mathcal{G}$ is an independent
sample of human respondents.
$\mathrm{RCG}>1$ indicates undercoverage of human respondents,
whereas $\mathrm{OSR}>1$ indicates generated respondents unusually far from the
human sample (\emph{off-support}).
A low OSR is not necessarily better: when $\mathrm{RCG}>1$ but $\mathrm{OSR}<1$,
generated respondents stay close to some human respondents while leaving
others poorly covered (\emph{contraction}).
\autoref{fig:rcg_osr_toy} illustrates these patterns:
the off-support and contraction examples have similar RCG values (2.17 and 2.06),
but their OSR values lie on opposite sides of one (2.24 and 0.55).

OSR has a similar form to med.DCR \citep{Lautrup2025-xj},
but addresses the opposite concern (\autoref{sec:compare_distributions}).
For privacy, a synthetic record that is too close to a human record poses a risk.
Accordingly, med.DCR divides the median synthetic-to-real nearest-neighbor distance by
its real-to-real counterpart, with values below one indicating greater concern.
In contrast, OSR uses the 95th percentile to identify generated respondents that are unusually far from the human sample, with values above one indicating off-support generation.

Retaining the distances also sets our ratios apart from the density and coverage of
\citet{Naeem2020-sg}, which test each nearest-neighbor distance against a threshold
(\autoref{sec:compare_distributions}).
Our ratios stay informative no matter how far apart the two sets are:
for a generated set that lies outside all human neighborhood balls,
the coverage of \citet{Naeem2020-sg} is zero however far away it lies, whereas $d^{95}_{\mathcal{H}\rightarrow\mathcal{G}}$ keeps growing with the distance.
Our ratios also do not require the neighborhood size $k$ that sets that threshold;
the only remaining choice is the aggregation percentile, whose effect we examine in Appendix~\ref{sec:appendix:sensitivity_of_percentile}.
\section{Evaluation Setup}\label{sec:setup}
\subsection{Datasets}\label{sec:datasets}
We use two surveys of US residents:
the World Values Survey (WVS; \citealp{wvs7}), on social, religious, and political values,
and the American National Election Studies (ANES; \citealp{anes}), on vote choice and partisanship.\footnote{
We use only human respondents who give valid answers to
at least 10 of the 100 questions (2,594 of 2,596 in WVS and 5,453 of 5,475 in ANES).}
For WVS, we use the United States sample of wave 7 (2,594 respondents)
and uniformly sample 100 questions from the 144 used by \citet{Liu2026-fl};
the dataset provides 17 attribute variables, such as sex, age, education, and income.
For ANES, we use the 2024 wave of its Time Series Cumulative Data File (February 2026 release; 5,453 respondents)
and uniformly sample 100 questions from the 149 remaining after excluding administrative,
demographic, and geographic items;
the dataset provides 11 attribute variables including party identification
and liberal--conservative ideology.\footnote{Party identification and ideology
are political rather than demographic, but we
condition on them following prior silicon-sampling studies on ANES
\citep{Argyle2023-kp,Bisbee2024-bw} and, for brevity, call all conditioning
variables demographic attributes.
}
Appendix~\ref{sec:appendix:demographic_attributes} lists all attributes and their binning.

In both datasets, the 100 questions are drawn once, and the same draw is used for all models and cells.
Of the 100 questions, 30 (WVS) and 32 (ANES) are numeric rating scales and
the rest are multiple-choice questions with ordinal or nominal options,
so each respondent is represented by a vector of $M = 100$ answers.

\subsection{Conditioning Cells}\label{sec:cells}
A \emph{cell} fixes the values of $L$ demographic attributes on which generation is conditioned
(\autoref{sec:conditional_generation}), e.g., (female, 30--39) with $L = 2$;
every cell also includes US residence, which all respondents share (\autoref{sec:datasets}).
We keep cells with at least 100 human respondents and no missing attribute values,
and sample up to 100 cells per $L$, which gives 730 cells for WVS and 603 for ANES
($L = 0$--$8$; Appendix~\ref{sec:appendix:cell}).
Each cell $c$ gives a pair $(\mathcal{H}_c, \mathcal{G}_c)$:
the answer vectors of the human respondents in $c$ and of the 100 generated respondents
whose prompts carry the attribute values of $c$ (\autoref{sec:prompting}).
All metrics in \autoref{sec:respondent_coverage} are computed per cell.

\subsection{Prompts}\label{sec:prompting}
Our prompt has four parts: a task instruction, a persona description, few-shot examples\footnote{
We use up to three question--answer pairs as examples
and evaluate the effect of including the full response history in Appendix~\ref{sec:appendix:full_history}.
} of the respondent's earlier answers, and the question with its choices.
To reduce sensitivity to wording and ordering,
we prepare three instruction formats and three persona formats following prior
studies~\citep{Argyle2023-kp,Santurkar2023-vx,Bisbee2024-bw,Taday-Morocho2026-yg,Moon2024-ql,AlKhamissi2024-tj},
and randomly assign one instruction format and one persona format to each generated respondent.
We also randomize the question order and the order of attributes in the persona.
For questions without a natural ordering of answer choices,
we shuffle the choices to reduce LLM preferences for particular answer labels~\citep{Zheng2024-kr,Dominguez-Olmedo2024-vn}.
The full templates and two example prompts are in Appendix~\ref{sec:appendix:prompt}.

\subsection{Models}\label{sec:models}
\paragraph{Proprietary.}
We use two proprietary models: gpt-5.4-nano (OpenAI) and 
gemini-3.1-flash-lite (Google, hereafter gemini-flash-lite). 
Additionally, in an ablation study (\autoref{sec:ablation}), we use gpt-5.4 and
gemini-3.1-pro-preview (Google, hereafter gemini-pro).
We leave the sampling parameters (temperature, top\_p, and top\_k where available) at their vendor
defaults and set only the reasoning effort to the lowest level available: \texttt{none} 
for the OpenAI models and \texttt{minimal} for the Google models.

\paragraph{Open-weight.}
We use four open-weight models: Qwen3.5-9B \citep[Alibaba;][]{qwen3.5},
gemma-4-12B-it~\citep[Google, hereafter gemma-4-12B;][]{gemmateam2026gemma4},
gpt-oss-20b~\cite{Agarwal2025-xs},
and Llama-3.3-70B-Instruct~\citep[Meta, hereafter Llama-3.3-70B;][]{Grattafiori2024-fh}.
We use their default decoding settings (Appendix~\ref{sec:appendix:model_parameters}) 
without tuning, to evaluate off-the-shelf use.

\subsection{Distance}\label{sec:distance}
Because the questions are of mixed types (\autoref{sec:datasets}), 
we use the Gower distance~\citep{Gower1971-kf} as $d(\bm{u}, \bm{v})$
(details in Appendix~\ref{sec:appendix:gower_distance}).
The Gower distance averages per-question distances, excluding the $m$-th question
when either $u_m$ or $v_m$ is missing.\footnote{
For generated respondents, we also treat an answer as missing when the response cannot be parsed,
for example, when the model refuses to answer.
We examine the effect of missing values in Appendix~\ref{sec:appendix:missing_value}.}

\subsection{Baselines}\label{sec:baselines}
We select three non-LLM baselines.
\emph{Marginals} samples each answer $u_m$ independently from the marginal distribution of the human respondents $\mathcal{H}$.
\emph{Copula} draws the whole answer vector $\bm{u}$ from a Gaussian copula~\citep{sklar1959fonctions}: 
the marginals are estimated as in Marginals, and the questions are coupled through a
correlation matrix estimated with Ledoit--Wolf shrinkage~\citep{Ledoit2004-ww},
as fitting uses fewer respondents than questions in about half of the cells.
\emph{LCA} draws $\bm{u}$ from a two-class latent class model~\citep{Goodman1974-jn}, that is, a mixture of two class-specific Marginals generators, fitted with StepMix~\citep{Morin2025-gd}.
Because these baselines are fitted to the same respondents against which they are evaluated,
they serve as references for what a generator can achieve with access to the cell’s respondents,
rather than as direct competitors to the LLMs (fitting details in Appendix~\ref{sec:appendix:hold_out}).

\section{Results}\label{sec:results}
\subsection{Main Results}\label{sec:main_results}
  \begin{table}[t]
  \centering
  \small
  \setlength{\tabcolsep}{2.5pt}
  \begin{tabular}{@{}lrrrrrr@{}}
  \toprule
   & \multicolumn{3}{c}{WVS} & \multicolumn{3}{c}{ANES} \\
  \cmidrule(lr){2-4} \cmidrule(lr){5-7}
   & RCG & Cal. & OSR & RCG & Cal. & OSR \\
  \midrule
  gpt-5.4-nano & 1.272 & 1.150 & 0.994 & 1.370 & 1.232 & 1.173 \\
  gemini-flash-lite & 1.218 & 1.126 & 0.840 & 1.353 & 1.178 & 0.838 \\
  Qwen3.5-9B & 1.273 & 1.145 & 0.992 & 1.347 & 1.217 & 1.213 \\
  gemma-4-12B & 1.199 & 1.137 & 0.793 & 1.332 & 1.194 & 0.942 \\
  gpt-oss-20b & 1.269 & 1.140 & 1.013 & 1.346 & 1.202 & 1.144 \\
  Llama-3.3-70B & 1.275 & 1.147 & 0.722 & 1.277 & 1.198 & 0.853 \\
  \midrule
  Marginals & 1.158 & \multicolumn{1}{c}{--} & 0.989 & 1.250 & \multicolumn{1}{c}{--} & 1.078 \\
  Copula & 1.041 & \multicolumn{1}{c}{--} & 0.946 & 1.091 & \multicolumn{1}{c}{--} & 0.976 \\
  LCA & 1.105 & \multicolumn{1}{c}{--} & 0.959 & 1.143 & \multicolumn{1}{c}{--} & 0.950 \\
  \bottomrule
  \end{tabular}
  \caption{The respondent coverage gap (RCG) and the off-support ratio (OSR) for
  each generator on each dataset. Cal.\ is the RCG after marginal calibration.
  OSR is reported before calibration.
}
  \label{table:main}
  \end{table}

For each cell, we compute RCG and OSR for 50 random draws of 100 human respondents.
We first average the scores over draws within each cell, then average the resulting cell-level scores across cells.
We also confirm that RCG behaves as intended;
in each cell $c$,\footnote{
We use the cell only if it contains at least 200 human respondents.
This leaves 203 of the 730 cells for WVS and 196 of the 603 cells for ANES.
}
we replace the generated set $\mathcal{G}_c$ with a separate sample of 100 human respondents
drawn from $\mathcal{H}_c$, disjoint from the 100 respondents used for evaluation,
and compute RCG exactly as above.
The average RCG is 1.001 (WVS) and 1.000 (ANES).
The same replacement gives an average OSR of 1.001 (WVS) and 1.002 (ANES).
Results under the hold-out fit setting are reported in Appendix~\ref{sec:appendix:hold_out}.

\autoref{table:main} shows the RCG and OSR for each generator and dataset (Appendix~\ref{sec:appendix:rcg} shows RCG on WVS and ANES by the number of conditioning attributes $L$).
All uncalibrated LLM RCG values in \autoref{table:main} exceed 1.0 and the values of all three baselines,
including the simplest baseline, Marginals.
For every LLM on both datasets, the 95\% confidence interval for the RCG difference
from Marginals excludes zero (\autoref{table:appendix:main_ci} in Appendix~\ref{sec:appendix:main_95}).
Thus no LLM covers the human respondents as well as another human sample does (1.000--1.001), or even
as well as Marginals, which ignores all dependence between questions.
RCG shows no consistent ordering by model family,
weight availability (proprietary vs. open-weight), or model size.

\paragraph{Calibrating marginal distributions.}
We next ask whether correcting per-question response distributions closes the respondent coverage gap.
We use the empirical human marginals in each evaluation cell as oracle targets,
applying quantile mapping~\citep{Gudmundsson2012-ik,Cannon2015-ba} to ordered responses and
minimum-rewrite remapping to nominal responses (Appendix~\ref{sec:appendix:calibration}).
Calibration (Cal.\ in \autoref{table:main}) reduces RCG for all six LLMs on both datasets.
The calibrated RCG values range from 1.126 to 1.150 on WVS and from 1.178 to 1.232 on ANES.
They are lower than those of Marginals but higher than the human reference and those of
both Copula and LCA.
All differences from the three baselines are statistically significant (\autoref{table:appendix:calibration_diff_ci} in Appendix~\ref{sec:appendix:main_95}).
Thus, this item-wise correction improves respondent coverage but does not restore human-level coverage.

\subsection{Failure Modes}\label{sec:failure_modes}
Using the uncalibrated OSR column of \autoref{table:main}, we next examine how each LLM fails to cover the human respondents.
The baselines (Marginals, Copula, and LCA) have OSR at or near one (0.946--1.078) on both datasets.
RCG alone does not show how a model fails: on ANES, the six LLMs span
1.277--1.370 in RCG but 0.838--1.213 in OSR.

In both datasets, OSR splits the six LLMs into two groups:
the upper group (gpt-oss-20b, gpt-5.4-nano, and Qwen3.5-9B) and 
the lower group (gemma-4-12B, gemini-flash-lite, and Llama-3.3-70B).
The OSRs of the lower group are below one (0.722--0.840 on WVS and 0.838--0.942 on ANES) and the RCGs are above one.
By the definitions in \autoref{sec:respondent_coverage}, this combination indicates contraction.
In the upper group, the OSRs on ANES are above one (1.144--1.213),
which is the off-support pattern.
By contrast, on WVS, the OSRs are close to one (0.992--1.013).

Prior work has reported both patterns under other names,
though not at the level of individual respondents’ joint answer profiles:
contraction as reduced variance~\citep{Bisbee2024-bw} and flattening~\citep{Wang2025-zs}; and
the off-support pattern as misportrayal~\citep{Wang2025-zs}.

\subsection{Prior Metrics Can Mask Coverage Failures}
\begin{figure}[t]
  \centering
  \includegraphics[width=\columnwidth]{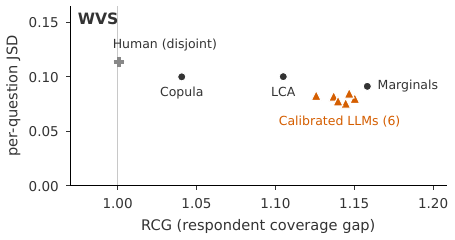}
  \caption{
Per-question JSD (mean over questions and cells) against RCG on WVS 
for three baselines and six LLMs.
For the LLMs, both metrics are computed after in-cell marginal calibration.
Human (disjoint) uses only the 203 cells with at least 200 human respondents.
Larger values are worse on both axes.
}
  \label{fig:calibrated_rcg_jsd_wvs}
\end{figure}
\begin{figure}[t]
  \centering
  \includegraphics[width=\columnwidth]{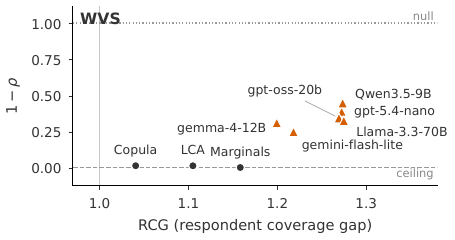}
  \caption{
  Correlation-structure mismatch ($1-\rho$) against RCG for all nine generators on WVS.
  Larger values are worse on both axes.
  Triangles are LLMs; circles are the statistical baselines.
  The dotted horizontal line marks the permutation null (each question column
  permuted across cells, preserving marginals but destroying correlations).
  The dashed horizontal line marks the split-half ceiling of the human data (\autoref{fn:ceiling}).
}
  \label{fig:williams_rcg}
\end{figure}
RCG is not a repackaging of existing metrics: marginal fidelity does not separate generators that differ in RCG, and an aggregate correlation metric is saturated by a generator that samples answers independently across questions.
We measure marginal fidelity by the mean per-question
Jensen–Shannon distance~\citep[JSD;][]{Durmus2024-hc} and
dependence across questions by the correlation-structure score of \citet{Williams2026-bs}.

\autoref{fig:calibrated_rcg_jsd_wvs} shows per-question JSD against RCG on WVS for the
human reference and for the generators that match the human answer distributions:
Marginals, Copula, LCA, and the six calibrated LLMs 
(the same figure for ANES is in Appendix~\ref{sec:appendix:prior_metrics}).
Their JSD values lie within a narrow range (0.08--0.11), whereas their RCG ranges from 1.04
(Copula) to 1.16 (Marginals).
The calibrated LLMs have the lowest JSD, as expected from calibrating to the human marginals
in each evaluation cell (\autoref{sec:main_results}),
yet their RCG remains above that of Copula and LCA.
Per-question JSD depends only on each question's answer distribution,
not on how a respondent's answers combine across questions.
RCG compares whole answer profiles and therefore reflects this combination.

\autoref{fig:williams_rcg} plots the correlation-structure mismatch
($1 - \rho$, where $\rho$ is the correlation-structure score of \citealp{Williams2026-bs})
against RCG for all generators on WVS
(the same figure for ANES is in Appendix~\ref{sec:appendix:prior_metrics}).
Marginals achieves $\rho = 0.994$ on WVS (split-half ceiling\footnote{
Following \citet{Williams2026-bs}, we split the human respondents of each cell into two random halves and compute $\rho$ between the two halves; the ceiling is the mean over 50 random splits.
\label{fn:ceiling}} of the human data: 0.995)
and $0.999$ on ANES (split-half ceiling of the human data: 0.999).
Its RCG, however, is $1.158$ on WVS and $1.250$ on ANES.

\begin{proposition}[Proof in Appendix~\ref{sec:appendix:proof_prop1}]\label{prop:williams}
For any set of generated respondents $\mathcal{G}_c$ in a cell $c$, let $\mathcal{G}'_c$ be
obtained by permuting, separately for each question $q_m$, the answers to $q_m$ among the
respondents of $\mathcal{G}_c$ who answered it.
Then $\mathcal{G}'_c$  has the same marginal for every question as $\mathcal{G}_c$,
and the correlation-structure score of \citet{Williams2026-bs} takes identical
values on $\{\mathcal{G}_c\}_c$ and $\{\mathcal{G}'_c\}_c$.
\end{proposition}
The permutation destroys the alignment of answers within respondents, yet the score is unchanged.
In particular, applying such within-cell permutations to the human
data yields a correlation-structure score of exactly one against the original human data.
This invariance is not a defect for measuring
aggregate cross-cell structure; rather, it shows that the score cannot
substitute for respondent-level evaluation.

\subsection{Number of Generated Respondents}\label{sec:ngen}
\begin{table}[t]
  \centering
  \small
  \setlength{\tabcolsep}{4pt}
   \begin{tabular}{lrrrr}
  \toprule
   & \multicolumn{3}{c}{RCG at $|\mathcal{G}_c| =$} & Slope \\
  \cmidrule(lr){2-4}
   & 100 & 1{,}000 & 10{,}000 & per decade \\
  \midrule
  gpt-5.4-nano & 1.270 & 1.197 & -- & $-0.073^{\dagger}$ \\
  gemini-flash-lite & 1.222 & 1.159 & -- & $-0.060^{\dagger}$ \\
  Qwen3.5-9B & 1.271 & 1.193 & 1.130 & $-0.070$ \\
  gemma-4-12B & 1.204 & 1.141 & 1.094 & $-0.056$ \\
  gpt-oss-20b & 1.267 & 1.189 & 1.124 & $-0.072$ \\
  \midrule
  Marginals & 1.161 & 1.074 & 1.003 & $-0.078$ \\
  Copula & 1.041 & 0.931 & 0.847 & $-0.095$ \\
  LCA & 1.109 & 1.023 & 0.954 & $-0.077$ \\
  \bottomrule
  \end{tabular}
  \caption{
  RCG (cell mean over the WVS subset of Section~\ref{sec:ablation})
  at three sizes of the generated set $|\mathcal{G}_c|$.
  Slope is the median over cells of the per-cell change in RCG per tenfold
  increase in $|\mathcal{G}_c|$.
  --: not generated; $\dagger$: fitted over one decade only, for reference.
  }
  \label{table:ngen_rcg_wvs}
\end{table}
As described in \autoref{sec:respondent_coverage}, RCG shrinks as $|\mathcal{G}_c|$ grows for any generator.
We ask how many generated respondents are needed for RCG to reach one.
We change $|\mathcal{G}_c|$ from 100 to 1{,}000 (10 $\times$ the 100 human respondents) for the two proprietary models
(gpt-5.4-nano and gemini-flash-lite)\footnote{API cost limits the proprietary models to 1{,}000.} and
to 10{,}000 (100 $\times$) for the three open-weight models (Qwen3.5-9B, gemma-4-12B, and gpt-oss-20b).\footnote{
Llama-3.3-70B was excluded due to computational cost.
}

\autoref{table:ngen_rcg_wvs} shows the RCG on the WVS subset of \autoref{sec:ablation}
at different numbers of generated respondents $|\mathcal{G}_c|$
and the slope of RCG regressed on $\log_{10}|\mathcal{G}_c|$ (the same table for ANES is in Appendix~\ref{sec:appendix:100x}). 
The RCG of the two proprietary models remains above one even at $|\mathcal{G}_c| = 1{,}000$,
and that of the three open-weight models remains above one even at $|\mathcal{G}_c| = 10{,}000$,
whereas that of Marginals falls to 1.003 at $|\mathcal{G}_c| = 10{,}000$.
OSRs change by less than 0.01 across different numbers of generated respondents (Appendix~\ref{sec:appendix:100x}).

Extrapolating the fitted trends linearly in $\log_{10}|\mathcal{G}_c|$
from $|\mathcal{G}_c|$ = 10{,}000 would place RCG = 1 at about $5 \times 10^5$ to $7 \times 10^5$
generated respondents per cell for the three open-weight models,
or 5{,}000 to 7{,}000 times the 100 human respondents.\footnote{Since the decrease per decade is already smaller in the second decade than in the first, these figures indicate the scale rather than the required sample size.}

\subsection{Robustness to Prompting, Inference Settings, and Model Tier}\label{sec:ablation}
We examine the robustness of the results in \autoref{sec:main_results} on a subset of the cells
in \autoref{sec:cells}: five cells per level of $L$, or all cells at levels with
fewer (41 cells for WVS and 37 for ANES).
The values in this subsection therefore differ from \autoref{table:main}.
Settings are in Appendix~\ref{sec:appendix:ablation}.

\paragraph{Sensitivity to temperature.}
We vary the temperature from 0.1 to 1.5.
RCG changes by at most 0.033 within each model, while OSR rises with temperature,
most for gpt-5.4-nano, Qwen3.5-9B, and gpt-oss-20b (Appendix~\ref{sec:appendix:temp}).

\paragraph{Sensitivity to reasoning effort.}
We also vary the reasoning effort, which controls how much the model reasons before answering,
over \texttt{low}, \texttt{medium}, and \texttt{high} for gpt-5.4-nano (default \texttt{none}) and
over \texttt{medium} and \texttt{high} for gpt-oss-20b (default \texttt{low}).
Raising the reasoning effort shifts both models toward contraction,
and the RCG of gpt-5.4-nano falls mostly from \texttt{none} to \texttt{low} (Appendix~\ref{sec:appendix:ret_effort}).

\paragraph{Sensitivity to model tier.}
We use gpt-5.4 from OpenAI, a higher-tier model of the same generation as gpt-5.4-nano,
and gemini-3.1-pro-preview (gemini-pro), a frontier model from Google.
Using higher-tier models reduces RCG except for gemini-pro on WVS,
and both shift toward contraction.
On ANES, gpt-5.4-nano at high effort and 
both higher-tier models have RCG point estimates 0.004–0.016 below Marginals
but remain above Copula and LCA (Appendices~\ref{sec:appendix:ret_effort} and  \ref{sec:appendix:ret_tier}).

\paragraph{Sensitivity to response history.}
We replace the up-to-three earlier answers in the prompt with
the respondent’s full response history (see \autoref{fn:history_size}).
RCG falls in all twelve model–dataset pairs, and in two of them (gemma-4-12B on WVS and Llama-3.3-70B on ANES) significantly below Marginals, but every pair remains above Copula (Appendix~\ref{sec:appendix:full_history}).
\section{Conclusions}\label{sec:conclusions}
We proposed the respondent coverage gap to evaluate silicon sampling at the respondent level.
No evaluated LLM achieves human-level coverage on either survey.
Under our main settings, all trail a baseline that samples answers independently across questions.
Across all tested temperatures, response histories, reasoning efforts,
model tiers, and generated sample sizes, every evaluated LLM has a
higher RCG than the Gaussian copula baseline at the same generated sample size.
Existing per-question metrics miss this failure.
Marginal calibration reduces RCG below that of Marginals, but not
that of Copula.
Before a synthetic population stands in for a human one, its respondents, not only its answers, should be shown to match.

\section*{Limitations}\label{sec:limitations}
\paragraph{Scope of the LLM analysis.}
RCG and OSR are defined only through distances between respondents'
answer profiles, so they do not depend on a domain or population.
Our analysis uses only US respondents from two surveys,
WVS wave 7~\citep{wvs7} and the 2024 wave of ANES~\citep{anes}.
Our findings on LLMs (RCG above one for all six models,
and the contraction and off-support patterns) are thus shown only for these surveys;
testing them in other areas requires further experiments.
Such areas include psychological experiments~\citep{Binz2025-rk},
online purchasing~\citep{Lu2026-xu}, and human--computer interaction~\citep{Hamalainen2023-mv}.
Non-US and multilingual settings~\citep{Ma2025-qr,Liu2026-fl} are also outside our analysis.

\paragraph{Evaluating alternative prompting approaches with RCG.}
We evaluate LLMs with off-the-shelf prompting, as practitioners typically use them.
Other prompting approaches, such as backstories~\citep{Moon2024-ql},
criteria-based prompting~\citep{Hayati2024-te}, Verbalized Sampling~\citep{Zhang2026-to},
and String Seed of Thought~\citep{Misaki2025-st}, are not examined.
Training-based methods, such as supervised
fine-tuning~\citep{Cao2025-bb,Suh2025-ve,Binz2025-rk,Krsteski2026-na}
and representation-level intervention~\citep{Wang2026-hj},
are also used to improve the accuracy and diversity of silicon sampling.
These methods target per-question accuracy or response distributions;
RCG can test whether their improvements extend to the level of individual respondents.

\section*{Ethical Considerations}\label{sec:ethical}
\paragraph{Potential risk of silicon sampling.}
Silicon sampling could misportray or flatten identity groups~\citep{Wang2025-zs} or
caricature them~\citep{Cheng2023-nq}, which may reinforce stereotypes when
synthetic populations replace human respondents.
Our metrics, RCG and OSR, can flag patterns consistent with off-support generation and contraction (\autoref{sec:failure_modes}) at the level of individual respondents.
However, the metrics are computed only on the evaluated questions and cells.

\paragraph{Data and terms of use.}
We use the WVS wave 7~\citep{wvs7} and ANES wave 2024~\citep{anes}.
We follow their terms of use: we use the data only for research,
do not redistribute the data files, and do not investigate specific survey respondents.

\paragraph{Use of AI.}
We used AI assistants to polish the manuscript's English
and assist with experimental code.
The authors reviewed and verified all manuscript content and code.

\bibliography{custom}

\appendix
\clearpage 
\section{Details of the Experiment in the Introduction}\label{sec:appendix:teaser}
In \autoref{sec:intro}, we use the same 2,594 respondents of WVS~\citep{wvs7}
and 100 questions as in the main text (\autoref{sec:datasets}).
We fit Marginals and Copula (\autoref{sec:baselines}) on all 2,594 human respondents and
generate 2,594 respondents from each for the scatter plots and the nearest-neighbor distances;
only the bar chart uses 50,000 generated respondents per generator.
For the top panels of \autoref{fig:teaser}, we apply principal coordinates analysis (PCoA) to
the Gower distance matrix of the pooled 7,782 points (Human, Marginals, and Copula),
treating the Gower distance as a squared distance.
We adopt the first two principal coordinates as the axes and plot 100 randomly chosen points per group.
In the bottom-right panel, the curves are kernel density estimates of the Gower distance
from each human respondent to the nearest other human respondent (Human)
and to the nearest generated respondent (Marginals and Copula),
as defined in \autoref{sec:respondent_coverage}.

\section{Demographic Attributes}\label{sec:appendix:demographic_attributes}
 \begin{table}[t]
  \centering
  \footnotesize
  \setlength{\tabcolsep}{4pt}
  \begin{tabular}{lcc}
  \toprule
  Attribute & WVS & ANES \\
  \midrule
  sex & 2 & 2 \\
  age (10-year bins; 80+ top-coded) & 6 & 7 \\
  education & 6 & 5 \\
  employment status & 6 & 4 \\
  income level & 8 & 5 \\
  marital status & 4 & 5 \\
  religion & 3 & 4 \\
  immigrant status & 2 & -- \\
  citizenship & 1 & -- \\
  chief wage earner & 2 & -- \\
  family savings & 4 & -- \\
  social class & 4 & -- \\
  occupational group & 6 & -- \\
  sector of employment & 3 & -- \\
  household size & 6 & -- \\
  number of children (0, 1, 2, 3, 4+) & 5 & -- \\
  living with parents & 2 & -- \\
  race & -- & 5 \\
  Hispanic origin & -- & 3 \\
  party identification (7-point) & -- & 7 \\
  liberal--conservative ideology & -- & 8 \\
  \midrule
  Number of attributes & 17 & 11 \\
  \bottomrule
  \end{tabular}
  \caption{Demographic attributes and their numbers of levels.}
  \label{table:all_demographic_attributes}
  \end{table}

\autoref{table:all_demographic_attributes} shows demographic attributes and their numbers of
levels.\footnote{
Levels with fewer than 100 respondents cannot form a cell and are not counted; missing values are also excluded.
}
We bin age and the number of children as shown in the table and keep the original categories of the other attributes.
As we described in \autoref{sec:datasets}, we treat party identification and
liberal--conservative ideology as demographic attributes for ANES.
Ideology has eight levels because we keep
``DK; haven't thought much about it'' as a level in addition to the seven-point scale.

\section{Prompts}\label{sec:appendix:prompt}
Each query is a single user message with four blocks separated
by blank lines: an instruction, a persona description, few-shot examples, and a question block.
The instruction and the persona description are each drawn from three variants;
the variant pair and the order of attributes in the persona description are assigned to
each generated respondent at random and kept fixed across its 100 questions.
Each generated respondent answers its 100 questions in a random order.
The few-shot examples carry up to three of the respondent's own previous answers,
drawn at random without replacement from the questions it has already
answered (\autoref{sec:prompting}).
\autoref{fig:appendix:prompt_examples} shows two full prompts from the main runs.

\begin{figure*}[p]
\begin{tcbraster}[raster columns=2, raster equal height=rows,
                  raster column skip=1em, raster valign=top]
\begin{promptbox}{Multiple choice / QA persona / i1}
\begin{PromptText}
You are a survey respondent. Answer the following World Values Survey (WVS) question from the point of view of the given persona.

It is 2017.
Interviewer: Were you born in this country or are you an immigrant to this country?
Me: Born in this country
Interviewer: What is your sex?
Me: Male
Interviewer: In what country do you live?
Me: The United States
Interviewer: What is your age in years?
Me: between 40 and 49

Interviewer: On this list are various groups of people. Could you please mention any that you would not like to have as neighbors? Unmarried couples living together
Me: Mentioned
Interviewer: Could you tell me how secure do you feel these days?
Me: Quite secure
Interviewer: Which, if any, of the following do you believe in? God
Me: Yes

Answer the following question with a single choice.

Choose one of the following numbers (1 2).
Output ONLY the answer tag in the form <ans>N</ans> (for example, <ans>1</ans>). Do NOT include any explanation, justification, reasoning, or surrounding text.

# Question

Which, if any, of the following do you believe in?

Heaven

# Choices

- 1: No
- 2: Yes

Respond with only <ans>N</ans>. No other text.
\end{PromptText}
\end{promptbox}
\begin{promptbox}{Rating scale / 2nd-person persona / i2}
\begin{PromptText}
Below you will be asked to provide a short description of your demographic information, and then answer some questions.

Description: It is 2017. You live in the United States. You were born in the United States. Your sex is Male. You are between 50 and 59 years old.

You answered "Yes" to the question "Which, if any, of the following do you believe in? God".
You answered "Rather important" to the question "For each of the following, indicate how important it is in your life. Would you say it is: Family".
You answered "Quite secure" to the question "Could you tell me how secure do you feel these days?".

Answer the following question with a single integer.

Answer with an integer between 1 (Never justifiable) and 10 (Always justifiable).
A larger number means closer to "Always justifiable".
If you cannot decide, answer 99.
Output ONLY the answer tag in the form <ans>N</ans> (for example, <ans>5</ans>). Do NOT include any explanation, justification, reasoning, or surrounding text.

# Question

Please tell me for each of the following actions whether you think it can always be justified, never be justified, or something in between, using this card.

Cheating on taxes if you have a chance

Respond with only <ans>N</ans>. No other text.
\end{PromptText}
\end{promptbox}
\end{tcbraster}
\caption{Two full prompts from the main run (WVS).
Left: a multiple-choice question with the QA persona format and instruction i1.
Right: a rating-scale question with the second-person persona format and instruction i2.
Both personas are conditioned on sex, immigrant status, and age,
and both prompts carry three few-shot examples taken from
the same generated respondent's previous answers.
Texts are reproduced verbatim, including the original WVS question wording.
The choice order in the left example is the shuffled order shown to the model.}
\label{fig:appendix:prompt_examples}
\end{figure*}

\paragraph{Instructions.}
We use three instruction texts from prior work, with the adjustments described below.
i1 and i2 are placed at the top of the prompt;
i3 is placed between the few-shot examples and the question block.
In i1, the survey name is ``World Values Survey (WVS)'' for WVS and
``American National Election Studies (ANES) survey'' for ANES.
For i3, we keep only the first two sentences of the original;
the remaining ones specify an answer format, which our question block replaces.

\begin{promptbox}[fontupper=\small]{Instruction Variants}
i1~\citep{Taday-Morocho2026-yg}: {\ttfamily You are a survey respondent.
Answer the following World Values Survey (WVS) question from the point of view of the given persona.}
\par\smallskip
i2~\citep{Moon2024-ql}: {\ttfamily Below you will be asked to provide a short description of your
demographic information, and then answer some questions.}
\par\smallskip
i3~\citep{AlKhamissi2024-tj}: {\ttfamily Answer the following question from this perspective.
Others will read what you choose; your goal is to convince them it was chosen
from the perspective of the persona described above.}
\end{promptbox}

\paragraph{Persona descriptions.}
We use three formats: an interview transcript (QA;~\citealp{Argyle2023-kp}),
a second-person description~\citep{Bisbee2024-bw},
and a first-person description~\citep{Argyle2023-kp,Santurkar2023-vx,Moon2024-ql}.
Residence in the United States is added to every persona as an additional attribute.
Every persona description starts with the survey year (``It is 2017.'' for WVS and ``It is 2024.'' for ANES), following \citet{Bisbee2024-bw}.
The following are the three formats for one cell (sex, immigrant status, and age)
with the same attribute order.

\begin{promptbox}{QA Format}
\begin{PromptText}
It is 2017.
Interviewer: What is your sex?
Me: Male
Interviewer: In what country do you live?
Me: The United States
Interviewer: Were you born in this country or are you an immigrant to this country?
Me: Born in this country
Interviewer: What is your age in years?
Me: between 40 and 49
\end{PromptText}
\end{promptbox}

\begin{promptbox}{Second-Person Format}
\begin{PromptText}
Description: It is 2017. Your sex is Male. You live in the United States. You were born in the United States. You are between 40 and 49 years old.
\end{PromptText}
\end{promptbox}

\begin{promptbox}{First-Person Format}
\begin{PromptText}
Description: It is 2017. My sex is Male. I live in the United States. I was born in the United States. I am between 40 and 49 years old.
\end{PromptText}
\end{promptbox}

\paragraph{Few-shot examples.}
The few-shot examples follow the persona format.
The \{answer\} slot holds the text of the selected option for
multiple-choice questions and the number for rating scales, including the don't-know code;
unparsable answers are skipped.
Line breaks inside a question text are collapsed into single spaces in few-shot examples.
This use of earlier answers as context follows \citet{Moon2024-ql}.

\begin{promptbox}{Few-Shot Example Formats}
\begin{PromptText}
Interviewer: {question}
Me: {answer}

You answered "{answer}" to the question "{question}".

I answered "{answer}" to the question "{question}".
\end{PromptText}
\end{promptbox}

\paragraph{Question block.}
The question block uses a fixed template around the question and its answer choices.
For multiple-choice questions, the numeric labels (1, 2, \ldots, K) remain fixed.
For each respondent–question pair, we shuffle the answer texts if the choices have no natural order; otherwise, we preserve their natural order.

\begin{promptbox}{Multiple-Choice Question Template}
\begin{PromptText}
Answer the following question with a single choice.

Choose one of the following numbers (1 2 ... K).
Output ONLY the answer tag in the form <ans>N</ans> (for example, <ans>1</ans>). Do NOT include any explanation, justification, reasoning, or surrounding text.

# Question

{question text}

# Choices

- 1: {answer text shown at position 1}
- ...
- K: {answer text shown at position K}

Respond with only <ans>N</ans>. No other text.
\end{PromptText}
\end{promptbox}

\begin{promptbox}{Rating-Scale Question Template}
\begin{PromptText}
Answer the following question with a single integer.

Answer with an integer between {min} ({label of min}) and {max} ({label of max}).
A larger number means closer to "{label of max}".
If you cannot decide, answer {code for don't know}.
Output ONLY the answer tag in the form <ans>N</ans> (for example, <ans>5</ans>). Do NOT include any explanation, justification, reasoning, or surrounding text.

# Question

{question text}

Respond with only <ans>N</ans>. No other text.
\end{PromptText}
\end{promptbox}

\section{Building Cells}\label{sec:appendix:cell}
As described in \autoref{sec:conditional_generation},
generation is conditioned on a set of demographic attributes.
We build these sets from the $V$ attribute variables of each dataset (\autoref{sec:datasets})
in three steps.
(i) We fix the number of conditioning attributes $L$ and enumerate all $\binom{V}{L}$
subsets of the variables. For $L = 2$ and the variables sex, age, and marital status,
the subsets are (sex, age), (sex, marital status), and (age, marital status). 
(ii) For each subset, every combination of attribute values observed in the data defines a \emph{cell},
e.g., (female, 30--39).
We keep a cell only if it contains at least 100 human respondents and none of its values is a missing code; the threshold matches the 100 respondents we generate per cell.
When a cell has more than 100 human respondents, we repeatedly draw 100 of them at random, 
so that \autoref{sec:respondent_coverage} always compares sets of equal size. 
(iii) For each $L$, we sample 100 cells uniformly at random from the retained cells,
or take all of them when fewer than 100 remain. All models are evaluated on the same cells. 
This yields 730 cells for WVS and 603 for ANES. 
Fewer than 100 cells remain at $L = 0$ (one cell), at $L = 1$ (70 and 55), and at the deepest levels
(WVS: 59 at $L = 8$; ANES: 46 at $L = 7$ and one at $L = 8$); every other level contributes 100 cells.\footnote{In both datasets no cell with nine or more attributes contains 100 human respondents.}
In total, this amounts to 13.33 million queries per model (one per respondent-question pair).

All respondents in both datasets reside in the United States
(\autoref{sec:datasets}), so every cell also includes this residence as an
attribute. Each cell $c$ gives one pair $(\mathcal{H}_c, \mathcal{G}_c)$:
$\mathcal{H}_c$ is the set of answer vectors of the human respondents in the
cell, and $\mathcal{G}_c$ is the set of 100 answer vectors generated with
the cell's attribute values in the prompt (\autoref{sec:prompting}). All
metrics in \autoref{sec:respondent_coverage} are computed per cell.

\section{LLM Settings and Computational Cost}\label{sec:appendix:model_parameters}
All queries to the proprietary models were issued between September 20 and October 1, 2026 (UTC).
We leave the sampling parameters (temperature, top\_p, and top\_k where available) at their vendor
defaults and set only the reasoning effort to the lowest level available: \texttt{none} 
for the OpenAI models and \texttt{minimal} for the Google models.

Open-weight models are served with vLLM 0.26.0~\citep{Kwon2023-om} in bfloat16 on eight NVIDIA H200 GPUs.
We send no sampling parameters, so each model runs with the values shipped in its generation config
(or vLLM's defaults when none are shipped), and we set reasoning to the lowest level available
(off for Qwen3.5-9B and gemma-4-12B, \texttt{low} for gpt-oss-20b).

\begin{table*}[t]
\centering
\footnotesize
\setlength{\tabcolsep}{4pt}
\begin{tabular}{llcccccl}
\toprule
Model & Weights / access & Params & Temp. & top-$p$ & top-$k$ & Reasoning & Run dates (2026, UTC) \\
\midrule
gpt-5.4-nano & OpenAI API & -- & 1.0 & 1.0 & -- & effort \texttt{none} & Sep 20 \\
gemini-flash-lite & Google Vertex AI & -- & 1.0 & 0.95$^\dagger$ & 64$^\dagger$ & thinking \texttt{minimal} & Sep 20 \\
Qwen3.5-9B & open weights, vLLM & 9B & 1.0 & 1.0 & -- & thinking off & Sep 20 \\
gemma-4-12B & open weights, vLLM & 12B & 1.0 & 0.95 & 64 & thinking off & Sep 20 -- 21 \\
gpt-oss-20b & open weights, vLLM & 20B & 1.0 & 1.0 & -- & effort \texttt{low} & Sep 20 \\
Llama-3.3-70B & open weights, vLLM & 70B & 0.6 & 0.9 & -- & -- & Sep 21 -- 24 \\
\bottomrule
\end{tabular}
\caption{Decoding settings of the six LLMs in the main experiment. --: not applicable or not disclosed. $\dagger$: assumed (see text).}\label{table:appendix:model_parameters}
\end{table*}

\autoref{table:appendix:model_parameters} lists the six LLMs and the decoding settings that
their defaults resolved to (\autoref{sec:models}).
For gemini-flash-lite, Vertex AI does not expose its defaults,
and we assume the values published for the same model in AI Studio ($\dagger$).
The exact identifiers are \texttt{gemini-3.1-flash-lite} for gemini-flash-lite and,
for the open-weight models, \texttt{Qwen/Qwen3.5-9B}, \texttt{google/gemma-4-12B-it},
\texttt{openai/gpt-oss-20b}, and \texttt{meta-llama/Llama-3.3-70B-Instruct}.
Every query is a single user message without a system prompt, with \texttt{max\_tokens} of 4,096.

\paragraph{Computational cost.}
For the runs reported in this paper,
we spent approximately \$11{,}160 on the proprietary models
(OpenAI \$4{,}140 and Google \$7{,}020), and
approximately 1{,}960 H200 GPU-hours (about 245 hours on eight GPUs) on the open-weight models.

\section{Gower Distance}\label{sec:appendix:gower_distance}
For two answer vectors $\bm{u}$ and $\bm{v}$ over the $M = 100$ questions, the Gower distance is
\begin{equation}
d(\bm{u}, \bm{v}) = \frac{\sum_{m=1}^{M} w_m\, d_m(u_m, v_m)}{\sum_{m=1}^{M} w_m},
\end{equation}
where $d_m$ depends on the type of question $m$.
For numeric and ordinal questions,
$d_m = \lvert u_m - v_m \rvert / R_m$, where $R_m$ is the range of the answer scale of question $m$:
the maximum minus the minimum for rating scales, and the number of choices minus one for
ordinal questions, whose answers are coded by their position in the choice order.
For nominal questions, $d_m = \mathbf{1}[u_m \neq v_m]$.
The weight $w_m$ is 1 if both respondents answered question $m$ and 0 otherwise,
so the distance is an average over the questions that both respondents answered.
We treat as missing the don't-know code offered for rating scales
(99 in WVS and $-1$ in ANES), the don't-know and missing codes of the survey data, and
any value outside the presented choices or scale.

Rating scales are treated as numeric, and multiple-choice questions as ordinal when their
choices have a natural order and as nominal otherwise,
giving 30, 45, and 25 questions of the three types in WVS and 32, 36, and 32 in ANES.
We check the type of every question by hand, reading its wording and choices.
In WVS, we treat the nine agree--disagree questions that offer ``Hard to say'' as nominal,
because ``Hard to say'' has no position on the agreement scale.

\paragraph{Near-empty respondents.}
A respondent who answered very few questions is at distance zero from everyone who
gave the same answers to those questions and can be the nearest neighbor of many respondents.
We therefore use only human respondents with at least 10 valid answers (\autoref{sec:datasets});
this removes 2 WVS and 22 ANES respondents from 6 and 42 cells,
and no cell falls below 100 human respondents.

\section{Ablation Settings}\label{sec:appendix:ablation}
The ablations in \autoref{sec:ablation} share the cell subset described there
(up to five cells per level of $L$; 41 cells for WVS and 37 for ANES).
The six models and the three baselines of the main experiment are not rerun;
their metrics are recomputed on this subset from the original responses.
The values therefore differ slightly from \autoref{table:main}.
Unless stated otherwise, all settings other than the varied one are those of the main experiment (Appendix~\ref{sec:appendix:model_parameters}).

\paragraph{Temperature.}
For the API models we send the temperature and top-$p$ explicitly:
gpt-5.4-nano at $t \in \{0.1, 0.5, 1.0, 1.5\}$ and
gemini-flash-lite at $t \in \{0.1, 1.0, 1.5\}$, both with top-$p = 1.0$.
For gemini-flash-lite we also rerun its default ($t = 1.0$, top-$p = 0.95$) and confirm that
the explicit setting reproduces the main run
(RCG and OSR differ by at most 0.006, about as much as between two runs of the same setting).
For the open-weight models we disable the released generation configurations and use the plain
vLLM sampler (top-$p = 1.0$, no top-$k$) at $t \in \{0.1, 0.5, 1.0, 1.5\}$ for Qwen3.5-9B,
gemma-4-12B, and Llama-3.3-70B, and at $t \in \{0.1, 1.0, 1.5\}$ for gpt-oss-20b.
All temperature runs are therefore untruncated except those of gemini-flash-lite,
which keep its default top-$k$.

\paragraph{Reasoning effort.}
We vary the reasoning effort of gpt-5.4-nano at \texttt{low}, \texttt{medium}, and \texttt{high} (default \texttt{none})
and of gpt-oss-20b at \texttt{medium} and \texttt{high} (default \texttt{low}).
Temperature and top-$p$ stay at their defaults, and \texttt{max\_tokens} is
raised to 16{,}384 for gpt-5.4-nano and 7{,}168 for gpt-oss-20b to leave room for reasoning tokens.

\paragraph{Model tier.}
We add gpt-5.4 and gemini-3.1-pro-preview (gemini-pro).
At \texttt{minimal}, the thinking level we use for both Google models, gemini-pro still uses about 260 thinking tokens
per call on average, whereas gemini-flash-lite uses none.

\section{Sensitivity to Percentile Choices}\label{sec:appendix:sensitivity_of_percentile}
  \begin{table*}[t]
  \centering
  \small
  \setlength{\tabcolsep}{4pt}
  \begin{tabular}{lrrrrrrrrrrrr}
  \toprule
   & \multicolumn{6}{c}{WVS} & \multicolumn{6}{c}{ANES} \\
  \cmidrule(lr){2-7} \cmidrule(lr){8-13}
   & $0.50$ & $0.75$ & $0.90$ & $0.95$ & $0.99$ & $1.00$ & $0.50$ & $0.75$ & $0.90$ & $0.95$ & $0.99$ & $1.00$ \\
  \midrule
  gpt-5.4-nano & 1.325 & 1.312 & 1.290 & 1.272 & 1.256 & 1.247 & 1.535 & 1.463 & 1.404 & 1.370 & 1.317 & 1.282 \\
  gemini-flash-lite & 1.157 & 1.185 & 1.210 & 1.218 & 1.230 & 1.231 & 1.293 & 1.331 & 1.346 & 1.353 & 1.361 & 1.359 \\
  Qwen3.5-9B & 1.311 & 1.307 & 1.290 & 1.273 & 1.255 & 1.237 & 1.510 & 1.436 & 1.377 & 1.347 & 1.296 & 1.259 \\
  gemma-4-12B & 1.144 & 1.167 & 1.188 & 1.199 & 1.210 & 1.204 & 1.370 & 1.356 & 1.342 & 1.332 & 1.321 & 1.306 \\
  gpt-oss-20b & 1.318 & 1.307 & 1.287 & 1.269 & 1.249 & 1.232 & 1.481 & 1.422 & 1.372 & 1.346 & 1.311 & 1.284 \\
  Llama-3.3-70B & 1.205 & 1.247 & 1.267 & 1.275 & 1.290 & 1.287 & 1.275 & 1.278 & 1.277 & 1.277 & 1.281 & 1.277 \\
  \midrule
  Marginals & 1.229 & 1.194 & 1.168 & 1.158 & 1.150 & 1.140 & 1.353 & 1.294 & 1.261 & 1.250 & 1.234 & 1.219 \\
  Copula & 1.082 & 1.061 & 1.046 & 1.041 & 1.034 & 1.036 & 1.099 & 1.091 & 1.089 & 1.091 & 1.098 & 1.102 \\
  LCA & 1.145 & 1.124 & 1.109 & 1.105 & 1.103 & 1.108 & 1.156 & 1.145 & 1.142 & 1.143 & 1.151 & 1.154 \\
  \bottomrule
  \end{tabular}
  \caption{RCGs at various percentiles $q$, $\tfrac{d^{q}_{\mathcal{H}\rightarrow\mathcal{G}}}{ d^{q}_{\mathcal{H}\rightarrow\mathcal{H}}}$.}
  \label{table:quantile_sensitivity_rcg}
  \end{table*}

  \begin{table*}[t]
  \centering
  \small
  \setlength{\tabcolsep}{4pt}
  \begin{tabular}{lrrrrrrrrrrrr}
  \toprule
   & \multicolumn{6}{c}{WVS} & \multicolumn{6}{c}{ANES} \\
  \cmidrule(lr){2-7} \cmidrule(lr){8-13}
   & $0.50$ & $0.75$ & $0.90$ & $0.95$ & $0.99$ & $1.00$ & $0.50$ & $0.75$ & $0.90$ & $0.95$ & $0.99$ & $1.00$ \\
  \midrule
  gpt-5.4-nano & 1.185 & 1.117 & 1.043 & 0.994 & 0.920 & 0.873 & 1.415 & 1.313 & 1.224 & 1.173 & 1.094 & 1.046 \\
  gemini-flash-lite & 0.929 & 0.898 & 0.867 & 0.840 & 0.796 & 0.764 & 0.966 & 0.912 & 0.865 & 0.838 & 0.794 & 0.766 \\
  Qwen3.5-9B & 1.171 & 1.106 & 1.038 & 0.992 & 0.921 & 0.877 & 1.442 & 1.347 & 1.262 & 1.213 & 1.135 & 1.090 \\
  gemma-4-12B & 0.937 & 0.884 & 0.829 & 0.793 & 0.738 & 0.704 & 1.123 & 1.047 & 0.980 & 0.942 & 0.884 & 0.848 \\
  gpt-oss-20b & 1.181 & 1.122 & 1.057 & 1.013 & 0.949 & 0.909 & 1.369 & 1.278 & 1.194 & 1.144 & 1.066 & 1.016 \\
  Llama-3.3-70B & 0.865 & 0.810 & 0.756 & 0.722 & 0.672 & 0.642 & 1.023 & 0.951 & 0.889 & 0.853 & 0.801 & 0.770 \\
  \midrule
  Marginals & 1.192 & 1.118 & 1.040 & 0.989 & 0.911 & 0.862 & 1.307 & 1.213 & 1.127 & 1.078 & 1.002 & 0.956 \\
  Copula & 1.079 & 1.034 & 0.982 & 0.946 & 0.887 & 0.850 & 1.072 & 1.047 & 1.005 & 0.976 & 0.926 & 0.893 \\
  LCA & 1.112 & 1.066 & 1.003 & 0.959 & 0.889 & 0.844 & 1.074 & 1.037 & 0.986 & 0.950 & 0.892 & 0.854 \\
  \bottomrule
  \end{tabular}
  \caption{OSRs at various percentiles $q$, $\tfrac{d^{q}_{\mathcal{G}\rightarrow\mathcal{H}}}{ d^{q}_{\mathcal{H}\rightarrow\mathcal{H}}}$.}
  \label{table:quantile_sensitivity_osr}
  \end{table*}

All six LLMs have cell-averaged RCG above one at every tested percentile.
We write the percentile in RCG and OSR as a fraction $q$; the main text uses $q = 0.95$, the 95th percentile.
The choice of $q$ determines which part of the nearest-neighbor distance distribution is emphasized.
At $q = 1.00$, the metric uses the maximum distance;
at $q=0.95$, it uses the 95th percentile and is less sensitive to the largest distances.

\autoref{table:quantile_sensitivity_rcg} and \autoref{table:quantile_sensitivity_osr} show RCG and OSR
for $q$ from 0.50 to 1.00.
Our main result, that every LLM has a higher RCG than Marginals (\autoref{table:main}),
holds in both datasets at $q = 0.90$, $0.95$, $0.99$, and $1.00$.
At $q \leq 0.75$, only gemini-flash-lite, gemma-4-12B, and Llama-3.3-70B have lower RCG than Marginals,
in three model--dataset pairs at $q = 0.75$ and five at $q = 0.50$.
These three models show the contraction pattern (\autoref{sec:failure_modes}):
their generated respondents stay close to some human respondents while leaving others poorly covered.
A small $q$ ignores the human respondents farthest from the generated ones,
which for these models are the poorly covered ones, so their RCG approaches
or falls below that of Marginals.
OSR decreases with $q$ for every generator, but the two groups of LLMs in \autoref{sec:failure_modes} do not overlap at any $q$.

\section{Baseline Fitting}\label{sec:appendix:hold_out}
\begin{table}[t]
  \centering
  \small
  \setlength{\tabcolsep}{3pt}
  \begin{tabular}{@{}lrrrr@{}}
  \toprule
   & \multicolumn{2}{c}{WVS} & \multicolumn{2}{c@{}}{ANES} \\
  \cmidrule(lr){2-3} \cmidrule(lr){4-5}
   & In-cell & Hold-out & In-cell & Hold-out \\
  \midrule
  gpt-5.4-nano & \multicolumn{2}{c}{1.261} & \multicolumn{2}{c@{}}{1.354} \\
  gemini-flash-lite & \multicolumn{2}{c}{1.221} & \multicolumn{2}{c@{}}{1.345} \\
  Qwen3.5-9B & \multicolumn{2}{c}{1.260} & \multicolumn{2}{c@{}}{1.326} \\
  gemma-4-12B & \multicolumn{2}{c}{1.186} & \multicolumn{2}{c@{}}{1.310} \\
  gpt-oss-20b & \multicolumn{2}{c}{1.255} & \multicolumn{2}{c@{}}{1.317} \\
  Llama-3.3-70B & \multicolumn{2}{c}{1.264} & \multicolumn{2}{c@{}}{1.273} \\
  \midrule
  Marginals & 1.156 & 1.165 & 1.263 & 1.272 \\
  Copula & 1.047 & 1.088 & 1.120 & 1.109 \\
  LCA & 1.108 & 1.128 & 1.168 & 1.151 \\
  \bottomrule
  \end{tabular}
  \caption{RCG under the In-cell and Hold-out settings. The LLMs are not fitted, so each LLM has one value for both settings.}
  \label{table:holdout}
  \end{table}

Copula transforms the answers to each question into normal scores using the ranks of their
codes (tied values are assigned average ranks), estimates the correlation matrix of these scores
with Ledoit--Wolf shrinkage, and maps sampled scores back to codes through
the standard normal cumulative distribution function
and the empirical quantile function of each question.
Nominal questions go through the same steps, so their codes are ranked in numerical order
although this order has no meaning.
For these questions, Copula preserves the fitted marginal answer distributions,
but the modeled dependence on other questions depends on this arbitrary ordering.
LCA models every question as categorical and does not use the order of the codes.

Copula and LCA are fitted to the respondents who answer every question.
In 147 of the 603 ANES cells, and in no WVS cell, fewer than 50 respondents do so;
there we fill each missing answer with the median of the codes observed for its question
in the cell (a placeholder with no meaning for nominal questions) and fit to all respondents.

In \autoref{sec:main_results}, we fit the baselines to the human respondents of each cell,
from which the 100 evaluated respondents are also drawn.
To check whether the comparison between the LLMs and the baselines depends on this,
we compare two settings on the cells with at least 200 human respondents
(203 of the 730 cells for WVS and 196 of the 603 cells for ANES):
\emph{In-cell} fits the baselines as in \autoref{sec:main_results};
\emph{Hold-out} fits them, in each of the 50 draws, to 100 human respondents
drawn from $\mathcal{H}_c$ without overlap with the 100 evaluation respondents.
Both settings use the same draws of the evaluated respondents.

\autoref{table:holdout} shows the results.
Compared with In-cell, Hold-out raises the RCG of Marginals by 0.008 on WVS and 0.009 on ANES,
and that of Copula and LCA on WVS by 0.040 and 0.020.
On ANES, Copula and LCA have lower RCG under Hold-out than under In-cell because 84\% of
their 100-respondent fitting samples require the median-imputation procedure described above.
This procedure allows the models to be fitted to all 100 respondents rather than only
those who answer every question.
Under Hold-out, every LLM still has a higher RCG than every baseline,
except that on ANES Llama-3.3-70B and Marginals have almost the same RCG (1.273 and 1.272).
Under In-cell, the RCG of Llama-3.3-70B is 0.010 above that of Marginals on the 196 ANES cells used here,
and 0.028 above it on all 603 ANES cells (\autoref{table:appendix:main_ci}).

\section{95\% Confidence Intervals for the Main Results}\label{sec:appendix:main_95}
  \begin{table}[t]
  \centering
  \small
  \setlength{\tabcolsep}{4pt}
  \begin{tabular}{@{}lcc@{}}
  \toprule
   & \multicolumn{2}{c}{RCG $-$ Marginals} \\
  \cmidrule(lr){2-3}
   & WVS & ANES \\
  \midrule
  gpt-5.4-nano & $+0.114$ & $+0.120$ \\[-2pt]
   & {\scriptsize $[+0.111, +0.117]$} & {\scriptsize $[+0.110, +0.130]$} \\
  gemini-flash-lite & $+0.060$ & $+0.103$ \\[-2pt]
   & {\scriptsize $[+0.055, +0.065]$} & {\scriptsize $[+0.088, +0.120]$} \\
  Qwen3.5-9B & $+0.115$ & $+0.097$ \\[-2pt]
   & {\scriptsize $[+0.112, +0.119]$} & {\scriptsize $[+0.087, +0.108]$} \\
  gemma-4-12B & $+0.041$ & $+0.083$ \\[-2pt]
   & {\scriptsize $[+0.038, +0.045]$} & {\scriptsize $[+0.067, +0.100]$} \\
  gpt-oss-20b & $+0.111$ & $+0.096$ \\[-2pt]
   & {\scriptsize $[+0.107, +0.115]$} & {\scriptsize $[+0.082, +0.111]$} \\
  Llama-3.3-70B & $+0.116$ & $+0.028$ \\[-2pt]
   & {\scriptsize $[+0.112, +0.120]$} & {\scriptsize $[+0.017, +0.039]$} \\
  \midrule
  Copula & $-0.118$ & $-0.159$ \\[-2pt]
   & {\scriptsize $[-0.121, -0.115]$} & {\scriptsize $[-0.172, -0.147]$} \\
  LCA & $-0.053$ & $-0.106$ \\[-2pt]
   & {\scriptsize $[-0.056, -0.051]$} & {\scriptsize $[-0.118, -0.095]$} \\
  \bottomrule
  \end{tabular}
\caption{Cell-wise differences of the uncalibrated RCG (RCG  in \autoref{table:main}) from Marginals, averaged over cells. Brackets show the 95\% CI from 10,000 bootstrap resamples of the sets of conditioning attributes, each with all of its cells. Every uncalibrated LLM is significantly
above Marginals on both datasets.}
  \label{table:appendix:main_ci}
  \end{table}

\begin{table*}[t]
\centering
\small
\setlength{\tabcolsep}{3pt}
\begin{tabular}{@{}lccc@{\hspace{12pt}}ccc@{}}
\toprule
 & \multicolumn{3}{c}{WVS} & \multicolumn{3}{c}{ANES} \\
\cmidrule(lr){2-4}\cmidrule(l){5-7}
 & $-$ Marginals & $-$ Copula & $-$ LCA
 & $-$ Marginals & $-$ Copula & $-$ LCA \\
\midrule
gpt-5.4-nano & $-0.008$ & $+0.110$ & $+0.045$ & $-0.018$ & $+0.141$ & $+0.089$ \\[-2pt]
 & {\scriptsize $[-0.009, -0.007]$} & {\scriptsize $[+0.107, +0.113]$} & {\scriptsize $[+0.043, +0.048]$}
 & {\scriptsize $[-0.020, -0.016]$} & {\scriptsize $[+0.130, +0.153]$} & {\scriptsize $[+0.078, +0.099]$} \\
gemini-flash-lite & $-0.032$ & $+0.085$ & $+0.021$ & $-0.072$ & $+0.087$ & $+0.035$ \\[-2pt]
 & {\scriptsize $[-0.034, -0.031]$} & {\scriptsize $[+0.083, +0.088]$} & {\scriptsize $[+0.019, +0.023]$}
 & {\scriptsize $[-0.078, -0.065]$} & {\scriptsize $[+0.080, +0.095]$} & {\scriptsize $[+0.028, +0.042]$} \\
Qwen3.5-9B & $-0.014$ & $+0.104$ & $+0.040$ & $-0.032$ & $+0.127$ & $+0.074$ \\[-2pt]
 & {\scriptsize $[-0.015, -0.012]$} & {\scriptsize $[+0.101, +0.107]$} & {\scriptsize $[+0.037, +0.042]$}
 & {\scriptsize $[-0.035, -0.030]$} & {\scriptsize $[+0.116, +0.138]$} & {\scriptsize $[+0.064, +0.084]$} \\
gemma-4-12B & $-0.021$ & $+0.096$ & $+0.032$ & $-0.055$ & $+0.104$ & $+0.051$ \\[-2pt]
 & {\scriptsize $[-0.023, -0.020]$} & {\scriptsize $[+0.094, +0.099]$} & {\scriptsize $[+0.030, +0.034]$}
 & {\scriptsize $[-0.060, -0.050]$} & {\scriptsize $[+0.096, +0.112]$} & {\scriptsize $[+0.044, +0.059]$} \\
gpt-oss-20b & $-0.019$ & $+0.099$ & $+0.035$ & $-0.047$ & $+0.112$ & $+0.059$ \\[-2pt]
 & {\scriptsize $[-0.020, -0.017]$} & {\scriptsize $[+0.096, +0.102]$} & {\scriptsize $[+0.032, +0.037]$}
 & {\scriptsize $[-0.052, -0.043]$} & {\scriptsize $[+0.103, +0.121]$} & {\scriptsize $[+0.051, +0.067]$} \\
Llama-3.3-70B & $-0.011$ & $+0.106$ & $+0.042$ & $-0.052$ & $+0.107$ & $+0.055$ \\[-2pt]
 & {\scriptsize $[-0.013, -0.010]$} & {\scriptsize $[+0.104, +0.109]$} & {\scriptsize $[+0.040, +0.044]$}
 & {\scriptsize $[-0.057, -0.047]$} & {\scriptsize $[+0.099, +0.116]$} & {\scriptsize $[+0.047, +0.062]$} \\
\bottomrule
\end{tabular}
\caption{Cell-wise differences of the calibrated RCG (Cal.\ in \autoref{table:main}) from each baseline, averaged over cells; e.g., ``$-$ Marginals'' is Cal.\ $-$ Marginals. Brackets show the 95\% CI from 10,000 bootstrap resamples of the sets of conditioning attributes, each with all of its cells, as in Table~\ref{table:appendix:main_ci}. Every calibrated LLM is significantly below Marginals and significantly above Copula and LCA on both datasets.}
\label{table:appendix:calibration_diff_ci}
\end{table*}

\autoref{table:appendix:main_ci} and \autoref{table:appendix:calibration_diff_ci} show
the 95\% confidence intervals of \autoref{sec:main_results}.

\section{RCG by \texorpdfstring{$L$}{L}}\label{sec:appendix:rcg}
\begin{figure}[t]
  \centering
  \includegraphics[width=\columnwidth]{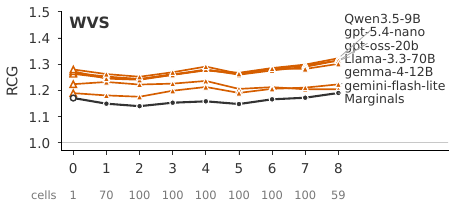}
  \caption{
  RCG on WVS by the number of conditioning attributes $L$.
  Each point is the mean over the cells at that $L$
  (cell counts below the axis; hollow markers denote a single cell).
  Triangles are LLMs and circles are Marginals;
  the gray horizontal line marks RCG $= 1$.}
  \label{fig:rcg_by_x_wvs}
\end{figure}

\begin{figure}[t]
  \centering
  \includegraphics[width=\columnwidth]{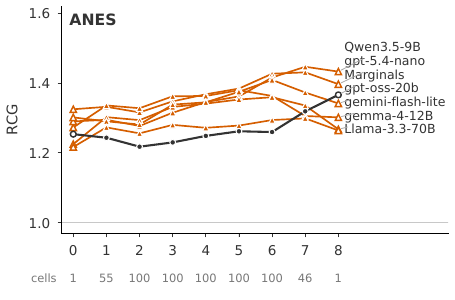}
  \caption{
  RCG on ANES by the number of conditioning attributes $L$.
  Each point is the mean over the cells at that $L$
  (cell counts below the axis; hollow markers denote a single cell).
  Triangles are LLMs and circles are Marginals;
  the gray horizontal line marks RCG $= 1$.}
  \label{fig:rcg_by_x_anes}
\end{figure}

\autoref{fig:rcg_by_x_wvs} and \autoref{fig:rcg_by_x_anes} show RCG on WVS and ANES 
by the number of conditioning attributes $L$.
At every value of $L$ on WVS, all six LLMs have a higher RCG than Marginals, and all
series stay above 1.
On WVS, cells at different values of $L$ represent different subpopulations,
whose median size decreases from 2,594 respondents at $L = 0$ to 114 at $L = 8$.
We therefore do not interpret the increase in RCG with $L$ as the
effect of adding an attribute.
Overall, on WVS the gap to Marginals is present at every conditioning depth we test, from the
unconditioned population ($L = 0$) to eight attributes; on ANES, it is present
from one to six attributes.

\section{Prior Metrics and RCG}\label{sec:appendix:prior_metrics}
  \begin{table}[t]
  \centering
  \small
  \setlength{\tabcolsep}{2.5pt}
  \begin{tabular}{@{}lrrrrrr@{}}
  \toprule
   & \multicolumn{3}{c}{WVS} & \multicolumn{3}{c}{ANES} \\
  \cmidrule(lr){2-4} \cmidrule(lr){5-7}
   & JSD & $1-\rho$ & RCG & JSD & $1-\rho$ & RCG \\
  \midrule
  gpt-5.4-nano & 0.403 & 0.392 & 1.272 & 0.409 & 0.494 & 1.370 \\
  gemini-flash-lite & 0.413 & 0.253 & 1.218 & 0.427 & 0.118 & 1.353 \\
  Qwen3.5-9B & 0.314 & 0.452 & 1.273 & 0.343 & 0.433 & 1.347 \\
  gemma-4-12B & 0.393 & 0.315 & 1.199 & 0.420 & 0.391 & 1.332 \\
  gpt-oss-20b & 0.358 & 0.348 & 1.269 & 0.372 & 0.503 & 1.346 \\
  Llama-3.3-70B & 0.452 & 0.329 & 1.275 & 0.448 & 0.267 & 1.277 \\
  \midrule
  Marginals & 0.091 & 0.006 & 1.158 & 0.097 & 0.001 & 1.250 \\
  Copula & 0.100 & 0.018 & 1.041 & 0.124 & 0.005 & 1.091 \\
  LCA & 0.100 & 0.018 & 1.105 & 0.125 & 0.006 & 1.143 \\
  \midrule
  Split-half ceiling & \multicolumn{1}{c}{--} & 0.005 & \multicolumn{1}{c}{--} & \multicolumn{1}{c}{--} & 0.001 & \multicolumn{1}{c}{--} \\
  Permutation null & \multicolumn{1}{c}{--} & 1.004 & \multicolumn{1}{c}{--} & \multicolumn{1}{c}{--} & 1.000 & \multicolumn{1}{c}{--} \\
  \bottomrule
  \end{tabular}
  \caption{JSD is the per-question Jensen--Shannon distance (base 2) averaged over questions and cells; $1-\rho$ is the correlation-structure mismatch,
  where $\rho$ is the correlation-structure score of \citet{Williams2026-bs}; RCG is as in Table~\ref{table:main}. Larger values are worse in every
  column. The last two rows give $1-\rho$ for the split-half ceiling of the human data and for the permutation null.}
  \label{table:appendix:prior_metrics}
  \end{table}

\autoref{table:appendix:prior_metrics} lists JSD, $1-\rho$, and RCG for all generators on both datasets.

\begin{figure}[p]
  \centering
  \includegraphics[width=\columnwidth]{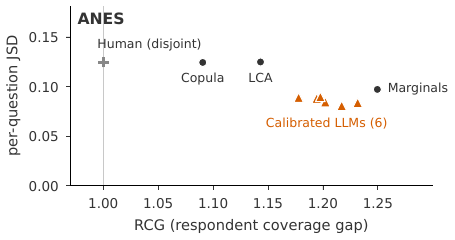}
  \caption{
Per-question JSD (mean over questions and cells) against RCG on ANES 
for three baselines and six LLMs.
For the LLMs, both metrics are computed after in-cell marginal calibration.
Human (disjoint) uses only the 196 cells with at least 200 human respondents.
Larger values are worse on both axes.
}
  \label{fig:jsd_rcg_anes}
\end{figure}

\autoref{fig:jsd_rcg_anes} shows mean per-question JSD against RCG for the baselines and
calibrated LLMs on ANES.

\begin{figure}[p]
  \centering
  \includegraphics[width=\columnwidth]{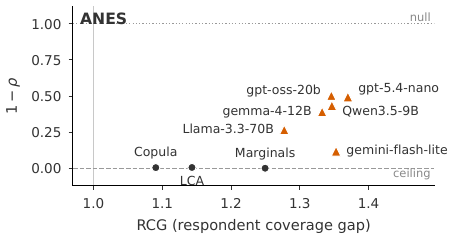}
  \caption{
  Correlation-structure mismatch ($1-\rho$) against RCG for all nine generators on ANES.
  Larger values are worse on both axes.
  Triangles are LLMs; circles are the statistical baselines.
  The dotted horizontal line marks the permutation null (each question column
  permuted across cells, preserving marginals but destroying correlations).
  The dashed horizontal line marks the split-half ceiling of the human data (\autoref{fn:ceiling}).
}
  \label{fig:williams_rcg_anes}
\end{figure}

\autoref{fig:williams_rcg_anes} shows correlation-structure mismatch ($1-\rho$) against RCG
for all nine generators on ANES.

\section{Proof of Proposition 1}\label{sec:appendix:proof_prop1}
The correlation-structure score of \citet{Williams2026-bs}
starts from a set of matrices $\{X_c\}_{c \in \mathcal{C}}$ with $X_c \in \mathbb{R}^{n_c \times M}$,
where $\mathcal{C}$ is the set of cells;
we denote by $X_{c,i,m}$ the answer of respondent $i$ in cell $c$ to $q_m$,
$i \in \mathcal{I}_c,\ |\mathcal{I}_c| = n_c$.
Let $X^{\mathcal{H}}_c$ be the matrix whose rows are the answer vectors in $\mathcal{H}_c$,
and $\mathcal{X}^{\mathcal{H}} = \{X^{\mathcal{H}}_c\}_{c \in \mathcal{C}}$;
we define $\mathcal{X}^{\mathcal{G}}$ from $\{\mathcal{G}_c\}_c$ in the same way,
and likewise $\mathcal{X}^{\mathcal{G}'}$ and $\mathcal{X}^{\mathcal{H}'}$ below.\footnote{
We use the cells as the subgroups of \citet{Williams2026-bs}.
}

The score of \citet{Williams2026-bs} is computed in three steps.
(i) For each cell $c$ and each of the $M_\mathrm{o}$ rating-scale or ordinal questions $q_m$,
let $A_{c,m}(\mathcal{X})$ be the mean answer to $q_m$ among the respondents of cell $c$ who answered it,
so that $A(\mathcal{X}) \in \mathbb{R}^{|\mathcal{C}| \times M_\mathrm{o}}$.\footnote{
Nominal questions are excluded because their answers have no mean.
}
(ii) Let $R(\mathcal{X}) \in \mathbb{R}^{M_\mathrm{o} \times M_\mathrm{o}}$ be the matrix of
Pearson correlations between the columns of $A(\mathcal{X})$, that is, between questions across cells,
and let $\mathbf{s}(\mathcal{X}) \in \mathbb{R}^{\frac{M_\mathrm{o}(M_\mathrm{o}-1)}{2}}$ be
the vector of its strictly upper-triangular entries.
(iii) The score is the Pearson correlation
$\rho = \operatorname{corr}\bigl(\mathbf{s}(\mathcal{X}^{\mathcal{H}}), \mathbf{s}(\mathcal{X}^{\mathcal{G}})\bigr)$.

\begin{proof}
In each cell $c$, the permutation only reorders the answers to $q_m$ among the respondents
in $\mathcal{G}_c$ who answered $q_m$, so the within-cell marginal of $q_m$ is unchanged,
which is the first claim.
In particular, the mean answer to each of the $M_\mathrm{o}$ questions is unchanged,
so $A(\mathcal{X}^{\mathcal{G}'}) = A(\mathcal{X}^{\mathcal{G}})$.
Since $R$, $\mathbf{s}$, and $\rho$ are computed from $A$ alone,
$\rho$ takes the same value on $\{\mathcal{G}_c\}_c$ and $\{\mathcal{G}'_c\}_c$.
\end{proof}

The same argument applied to the human respondents
gives $A(\mathcal{X}^{\mathcal{H}'}) = A(\mathcal{X}^{\mathcal{H}})$,
so the score of $\{\mathcal{H}'_c\}_c$ against $\{\mathcal{H}_c\}_c$ is the correlation of
$\mathbf{s}(\mathcal{X}^{\mathcal{H}})$ with itself, which is one.

\section{Number of Generated Respondents}\label{sec:appendix:100x}
\begin{table}[t]
  \centering
  \small
  \setlength{\tabcolsep}{4pt}
  \begin{tabular}{lrrrr}
  \toprule
   & \multicolumn{3}{c}{RCG at $|\mathcal{G}_c| =$} & Slope \\
  \cmidrule(lr){2-4}
   & 100 & 1{,}000 & 10{,}000 & per decade \\
  \midrule
  gpt-5.4-nano & 1.391 & 1.298 & -- & $-0.093^{\dagger}$ \\
  gemini-flash-lite & 1.359 & 1.259 & -- & $-0.097^{\dagger}$ \\
  Qwen3.5-9B & 1.357 & 1.239 & 1.146 & $-0.104$ \\
  gemma-4-12B & 1.338 & 1.242 & 1.167 & $-0.083$ \\
  gpt-oss-20b & 1.345 & 1.240 & 1.159 & $-0.092$ \\
  \midrule
  Marginals & 1.271 & 1.166 & 1.080 & $-0.094$ \\
  Copula & 1.086 & 0.964 & 0.874 & $-0.103$ \\
  LCA & 1.143 & 1.043 & 0.965 & $-0.089$ \\
  \bottomrule
  \end{tabular}
  \caption{
    RCG (cell mean over the ANES subset of Section~\ref{sec:ablation})
  at three sizes of the generated set $|\mathcal{G}_c|$.
  Slope is the median over cells of the per-cell change in RCG per tenfold
  increase in $|\mathcal{G}_c|$.
  --: not generated; $\dagger$: fitted over one decade only, for reference.
  }
  \label{table:ngen_rcg_anes}
  \end{table}

 \begin{table*}[t]
  \centering
  \small
  \setlength{\tabcolsep}{4pt}
  \begin{tabular}{lrrrrrrrr}
  \toprule
   & \multicolumn{4}{c}{WVS} & \multicolumn{4}{c}{ANES} \\
  \cmidrule(lr){2-5} \cmidrule(lr){6-9}
   & \multicolumn{3}{c}{OSR at $|\mathcal{G}_c| =$} & Slope & \multicolumn{3}{c}{OSR at $|\mathcal{G}_c| =$} & Slope \\
  \cmidrule(lr){2-4} \cmidrule(lr){6-8}
   & 100 & 1{,}000 & 10{,}000 & per decade & 100 & 1{,}000 & 10{,}000 & per decade \\
  \midrule
  gpt-5.4-nano & 0.988 & 0.992 & -- & $0.003^{\dagger}$ & 1.193 & 1.197 & -- & $0.005^{\dagger}$ \\
  gemini-flash-lite & 0.828 & 0.832 & -- & $0.004^{\dagger}$ & 0.870 & 0.875 & -- & $0.004^{\dagger}$ \\
  Qwen3.5-9B & 0.981 & 0.986 & 0.986 & $0.002$ & 1.232 & 1.237 & 1.238 & $0.002$ \\
  gemma-4-12B & 0.783 & 0.787 & 0.787 & $0.002$ & 0.967 & 0.972 & 0.972 & $0.002$ \\
  gpt-oss-20b & 1.008 & 1.013 & 1.013 & $0.002$ & 1.163 & 1.168 & 1.169 & $0.002$ \\
  \midrule
  Marginals & 0.990 & 0.993 & 0.993 & $0.002$ & 1.106 & 1.109 & 1.109 & $0.002$ \\
  Copula & 0.946 & 0.950 & 0.951 & $0.002$ & 0.999 & 1.005 & 1.006 & $0.003$ \\
  LCA & 0.963 & 0.967 & 0.967 & $0.002$ & 0.961 & 0.966 & 0.966 & $0.002$ \\
  \bottomrule
  \end{tabular}
  \caption{
  OSR (cell mean over the subsets of Appendix~\ref{sec:appendix:ablation}) at three sizes of the generated set $|\mathcal{G}_c|$.
  Slope is the median over cells of the per-cell slope of OSR regressed on $\log_{10} |\mathcal{G}_c|$,
  i.e., the change per tenfold increase. The two proprietary models were generated only up to
  $|\mathcal{G}_c| = 1{,}000$ (--: not generated); their slopes ($\dagger$) are fitted over one decade
  instead of two and are reported for reference.
  }
  \label{table:ngen_osr}
  \end{table*}

This section reports the ANES results and the OSR for \autoref{sec:ngen}.
At each $|\mathcal{G}_c|$, the generated set is drawn at random from
1,000 (proprietary models) or 10,000 (others) respondents per cell,
generated separately from the main experiment, so the values at
$|\mathcal{G}_c| = 100$ in \autoref{table:ngen_rcg_wvs}, \autoref{table:ngen_rcg_anes},
and \autoref{table:ngen_osr} differ from those of the main experiment on the
same subset (Appendix~\ref{sec:appendix:ablation}) by less than 0.01.

\autoref{table:ngen_rcg_anes} is the ANES counterpart of \autoref{table:ngen_rcg_wvs}.
RCG starts higher on ANES and falls faster with $|\mathcal{G}_c|$ for every generator.
Still, as on WVS, no LLM reaches one within the sizes we generate.
We estimate how many generated respondents per cell are needed
for RCG to reach one by extrapolating the slopes in \autoref{table:ngen_rcg_anes} linearly in
$\log_{10}|\mathcal{G}_c|$ from $|\mathcal{G}_c| = 10{,}000$.
The estimates are about $2.5 \times 10^5$ to $1.0 \times 10^6$ generated respondents per cell for the three open-weight models, that is, about 2{,}500 to 10{,}000 times the 100 human respondents.
As on WVS, RCG falls less in the second decade than in the first, so these estimates are likely underestimates.

\autoref{table:ngen_osr} shows OSR at the same sizes on both datasets.
From $|\mathcal{G}_c| = 100$ to the largest size we generate, OSR rises by 0.003 to 0.007 for every generator.

\section{Full-History Ablation}\label{sec:appendix:full_history}
 \begin{table*}[t]
  \centering
  \small
  \setlength{\tabcolsep}{4pt}
  \begin{tabular}{@{}llrrrrrr@{}}
  \toprule
   & & \multicolumn{3}{c}{WVS} & \multicolumn{3}{c}{ANES} \\
  \cmidrule(lr){3-5} \cmidrule(lr){6-8}
   & & RCG & RCG $-$ Marginals & OSR & RCG & RCG $-$ Marginals & OSR \\
  \midrule
  gpt-5.4-nano & limited history $(\leq 3)$ & 1.270 & $+0.109$ {\scriptsize $[+0.097, +0.121]$} & 0.990 & 1.390 & $+0.125$ {\scriptsize $[+0.087, +0.165]$} & 1.189 \\
   & full history & 1.231 & $+0.070$ {\scriptsize $[+0.056, +0.083]$} & 1.124 & 1.328 & $+0.062$ {\scriptsize $[+0.015, +0.113]$} & 1.304 \\
  \addlinespace[2pt]
  gemini-flash-lite & limited history $(\leq 3)$ & 1.218 & $+0.057$ {\scriptsize $[+0.038, +0.076]$} & 0.829 & 1.356 & $+0.090$ {\scriptsize $[+0.039, +0.141]$} & 0.878 \\
   & full history & 1.184\rlap{$^{\dagger}$} & $+0.022$ {\scriptsize $[-0.002, +0.046]$} & 0.901 & 1.249\rlap{$^{\dagger}$} & $-0.017$ {\scriptsize
  $[-0.078, +0.048]$} & 0.820 \\
  \addlinespace[2pt]
  Qwen3.5-9B & limited history $(\leq 3)$ & 1.268 & $+0.107$ {\scriptsize $[+0.094, +0.120]$} & 0.981 & 1.354 & $+0.089$ {\scriptsize $[+0.053, +0.129]$} & 1.234 \\
   & full history & 1.258 & $+0.097$ {\scriptsize $[+0.082, +0.111]$} & 1.097 & 1.276\rlap{$^{\dagger}$} & $+0.010$ {\scriptsize $[-0.036, +0.061]$} &
  1.326 \\
  \addlinespace[2pt]
  gemma-4-12B & limited history $(\leq 3)$ & 1.205 & $+0.044$ {\scriptsize $[+0.031, +0.058]$} & 0.783 & 1.332 & $+0.066$ {\scriptsize $[+0.014, +0.124]$} & 0.968 \\
   & full history & 1.113\rlap{$^{\ddagger}$} & $-0.049$ {\scriptsize $[-0.063, -0.035]$} & 0.837 & 1.224\rlap{$^{\dagger}$} & $-0.042$ {\scriptsize
  $[-0.103, +0.025]$} & 1.033 \\
  \addlinespace[2pt]
  gpt-oss-20b & limited history $(\leq 3)$ & 1.268 & $+0.107$ {\scriptsize $[+0.093, +0.121]$} & 1.003 & 1.338 & $+0.072$ {\scriptsize $[+0.027, +0.123]$} & 1.163 \\
   & full history & 1.193 & $+0.031$ {\scriptsize $[+0.017, +0.046]$} & 1.054 & 1.276\rlap{$^{\dagger}$} & $+0.010$ {\scriptsize $[-0.039, +0.063]$} &
  1.251 \\
  \addlinespace[2pt]
  Llama-3.3-70B & limited history $(\leq 3)$ & 1.279 & $+0.118$ {\scriptsize $[+0.102, +0.135]$} & 0.709 & 1.274\rlap{$^{\dagger}$} & $+0.008$ {\scriptsize $[-0.031,
  +0.048]$} & 0.877 \\
   & full history & 1.243 & $+0.082$ {\scriptsize $[+0.063, +0.101]$} & 0.724 & 1.156\rlap{$^{\ddagger}$} & $-0.110$ {\scriptsize $[-0.160, -0.056]$} &
  0.938 \\
  \midrule
  Marginals & & 1.161 &  & 0.991 & 1.266 &  & 1.108 \\
  Copula & & 1.040 &  & 0.947 & 1.086 &  & 0.991 \\
  LCA & & 1.111 &  & 0.964 & 1.149 &  & 0.964 \\
  \bottomrule
  \end{tabular}
 \caption{RCG $-$ Marginals is the cell-wise difference in RCG from Marginals, averaged over cells.
  The brackets show its 95\% CI from 10{,}000 bootstrap resamples of the sets of conditioning attributes, each with all of its cells.
  The CI thus covers variation from the choice of attribute sets, with the human and generated responses fixed.
  $^{\dagger}$ indicates that the 95\% CI of RCG $-$ Marginals includes zero, 
  whereas $^{\ddagger}$ indicates that the interval lies entirely below zero.
  With the full history, all RCG values remain above those of Copula.
  }
  \label{table:full_history_ci}
  \end{table*}

Providing the full response history lowers mean RCG in all twelve model–dataset pairs, 
although every mean RCG remains above one.

We check how the amount of response history included in the prompt affects RCG and OSR.
We use the full history as few-shot examples (\autoref{table:full_history_ci}).
$^{\dagger}$ indicates that the 95\% CI of RCG $-$ Marginals includes zero, 
whereas $^{\ddagger}$ indicates that the interval lies entirely below zero.
All full-history RCG values are above those of Copula;
the RCGs of four of six LLMs for WVS and one of six
for ANES are significantly above Marginals.
Only gemma-4-12B on WVS and Llama-3.3-70B on ANES have RCG values significantly below those of Marginals.

\section{Calibration Settings}\label{sec:appendix:calibration}
Let $a$ and $b$ denote answer choices for question $q_m$ and $p^{\mathcal{H}}_{c, m}$ be the
distribution of answers to $q_m$ among respondents in $\mathcal{H}_c$
(and $p^{\mathcal{G}}_{c,m}$ likewise for $\mathcal{G}_c$).
We consider a transition matrix $T_{c,m}(b \mid a)$ that satisfies
$\sum_a T_{c,m}(b \mid a)p^{\mathcal{G}}_{c, m}(a) = p^{\mathcal{H}}_{c, m}(b)$ for all $b$
and $\sum_b T_{c,m}(b \mid a) = 1$ for all $a$.
We replace each answer $a$ of a generated respondent with a draw from $T_{c,m}(\cdot \mid a)$,
independently across questions and respondents, and leave missing answers as they are.

The constraint leaves $T_{c,m}$ underdetermined.
For ordered questions (rating scales and ordinal questions), we use quantile
mapping~\citep{Gudmundsson2012-ik,Cannon2015-ba}, which is an optimal transport plan
for the cost $\lvert a-b \rvert$~\citep{Santambrogio2015-sd}.
With choices $\{1,2,3\}$, $p^{\mathcal{G}} = (0.5, 0.5, 0)$, and $p^{\mathcal{H}} = (0.25, 0.25, 0.5)$,
it moves half of the answers at 1 to 2 and all answers at 2 to 3.
For nominal questions, where $\lvert a-b \rvert$ has no meaning, 
we use minimum-rewrite remapping,
an optimal transport plan for the cost $\mathbf{1}[a \neq b]$~\citep{Levin2017-mc}:
it preserves a probability mass of $\min(p^G_{c,m}(a),p^H_{c,m}(a))$ at each choice $a$ and
reallocates the excess to choices where $p^G_{c,m}$ falls short of $p^H_{c,m}$.

\section{Missing-Value Ablation}\label{sec:appendix:missing_value}
  \begin{table}[t]
  \centering
  \small
  \setlength{\tabcolsep}{3pt}
  \begin{tabular}{@{}llrrrr@{}}
  \toprule
   & & \multicolumn{2}{c}{WVS} & \multicolumn{2}{c}{ANES} \\
  \cmidrule(lr){3-4} \cmidrule(lr){5-6}
   & & RCG & OSR & RCG & OSR \\
  \midrule
  gpt-5.4-nano & original & 1.272 & 0.994 & 1.370 & 1.173 \\
   & masked & 1.265 & 0.992 & 1.337 & 1.169 \\
  \addlinespace[2pt]
  gemini-flash-lite & original & 1.218 & 0.840 & 1.353 & 0.838 \\
   & masked & 1.209 & 0.840 & 1.305 & 0.833 \\
  \addlinespace[2pt]
  Qwen3.5-9B & original & 1.273 & 0.992 & 1.347 & 1.213 \\
   & masked & 1.267 & 0.992 & 1.318 & 1.210 \\
  \addlinespace[2pt]
  gemma-4-12B & original & 1.199 & 0.793 & 1.332 & 0.942 \\
   & masked & 1.191 & 0.794 & 1.287 & 0.938 \\
  \addlinespace[2pt]
  gpt-oss-20b & original & 1.269 & 1.013 & 1.346 & 1.144 \\
   & masked & 1.263 & 1.011 & 1.314 & 1.143 \\
  \addlinespace[2pt]
  Llama-3.3-70B & original & 1.275 & 0.722 & 1.277 & 0.853 \\
   & masked & 1.261 & 0.724 & 1.237 & 0.851 \\
  \midrule
  Marginals & original & 1.158 & 0.989 & 1.250 & 1.078 \\
   & masked & 1.154 & 0.991 & 1.222 & 1.075 \\
  \addlinespace[2pt]
  Copula & original & 1.041 & 0.946 & 1.091 & 0.976 \\
   & masked & 1.037 & 0.947 & 1.063 & 0.972 \\
  \addlinespace[2pt]
  LCA & original & 1.105 & 0.959 & 1.143 & 0.950 \\
   & masked & 1.101 & 0.960 & 1.111 & 0.947 \\
  \bottomrule
  \end{tabular}
  \caption{RCG and OSR in the original setting (\autoref{table:main}) and the masked setting.}
  \label{table:missing_mask}
  \end{table}

We evaluate the impact of missing values on the RCG and OSR.
In calculating RCG and OSR, the Gower distance excludes questions with missing answers;
fewer compared questions make distances more variable, so nearest-neighbor distances become smaller.
Human respondents have missing answers while generated respondents rarely do,
so human-to-human distances are computed over fewer questions than
human-to-generated distances;
this asymmetry may reduce the human-to-human distance and thereby inflate RCG.

To check this, we compute a \emph{masked} variant:
in calculating RCG and OSR, we repeat the following process for every generated respondent $\bm{v} \in \mathcal{G}_c$.
We randomly select a human respondent $\bm{u} \in \mathcal{H}_c$
outside the 100 human respondents used in the calculation, and
take the questions that $\bm{u}$ left unanswered as a \emph{mask}.
We then apply the mask to $\bm{v}$ by treating the answers
of $\bm{v}$ to these questions as missing.
In cells with exactly 100 human respondents
(11 of the 730 cells for WVS and 8 of the 603 cells for ANES), no human respondent remains
outside the evaluation sample.
For these cells, we randomly assign the 100 human respondents' missingness
masks to the 100 generated respondents in a one-to-one mapping.
\autoref{table:missing_mask} therefore covers the same cells as
\autoref{table:main}.
This procedure assesses sensitivity to the mismatch in missingness patterns between human and generated respondents.

\autoref{table:missing_mask} shows the results.
The masked RCGs are lower than the original RCGs, but 
the masked RCGs of all six LLMs remain above Marginals.
RCG decreases more on ANES (by 0.03--0.05) than on WVS (by at most 0.014), because
human respondents leave more answers missing on ANES (6.0\% vs. 1.1\%).
OSR changes by less than 0.01.
 
\section{Temperature Ablation}\label{sec:appendix:temp}
  \begin{table}[t]
  \centering
  \small
  \setlength{\tabcolsep}{4pt}
  \begin{tabular}{lrrrr}
  \toprule
   & \multicolumn{2}{c}{WVS} & \multicolumn{2}{c}{ANES} \\
  \cmidrule(lr){2-3} \cmidrule(lr){4-5}
   & RCG & OSR & RCG & OSR \\
  \midrule
  \multicolumn{5}{l}{gpt-5.4-nano} \\
  \quad $t = 0.1$ & 1.270 & 0.939 & 1.403 & 1.136 \\
  \quad $t = 0.5$ & 1.270 & 0.955 & 1.395 & 1.152 \\
  \quad $t = 1.0$ (default) & 1.270 & 0.990 & 1.390 & 1.189 \\
  \quad $t = 1.5$ & 1.271 & 1.044 & 1.394 & 1.251 \\
  \midrule
  \multicolumn{5}{l}{gemini-flash-lite} \\
  \quad default & 1.218 & 0.829 & 1.356 & 0.878 \\
  \quad $t = 0.1$ & 1.221 & 0.827 & 1.354 & 0.868 \\
  \quad $t = 1.0$ & 1.219 & 0.830 & 1.346 & 0.873 \\
  \quad $t = 1.5$ & 1.216 & 0.838 & 1.349 & 0.878 \\
  \midrule
  \multicolumn{5}{l}{Qwen3.5-9B} \\
  \quad $t = 0.1$ & 1.290 & 0.812 & 1.339 & 1.107 \\
  \quad $t = 0.5$ & 1.280 & 0.860 & 1.342 & 1.147 \\
  \quad $t = 1.0$ (default) & 1.268 & 0.981 & 1.354 & 1.234 \\
  \quad $t = 1.5$ & 1.257 & 1.087 & 1.346 & 1.302 \\
  \midrule
  \multicolumn{5}{l}{gemma-4-12B} \\
  \quad default & 1.205 & 0.783 & 1.332 & 0.968 \\
  \quad $t = 0.1$ & 1.203 & 0.776 & 1.339 & 0.963 \\
  \quad $t = 0.5$ & 1.204 & 0.778 & 1.340 & 0.962 \\
  \quad $t = 1.0$ & 1.200 & 0.786 & 1.335 & 0.970 \\
  \quad $t = 1.5$ & 1.205 & 0.801 & 1.331 & 0.983 \\
  \midrule
  \multicolumn{5}{l}{gpt-oss-20b} \\
  \quad $t = 0.1$ & 1.261 & 0.970 & 1.359 & 1.161 \\
  \quad $t = 1.0$ (default) & 1.268 & 1.003 & 1.338 & 1.163 \\
  \quad $t = 1.5$ & 1.285 & 1.076 & 1.355 & 1.196 \\
  \midrule
  \multicolumn{5}{l}{Llama-3.3-70B} \\
  \quad default & 1.279 & 0.709 & 1.274 & 0.877 \\
  \quad $t = 0.1$ & 1.277 & 0.708 & 1.277 & 0.876 \\
  \quad $t = 0.5$ & 1.276 & 0.709 & 1.277 & 0.877 \\
  \quad $t = 1.0$ & 1.278 & 0.712 & 1.277 & 0.877 \\
  \quad $t = 1.5$ & 1.276 & 0.716 & 1.276 & 0.881 \\
  \midrule
  Marginals & 1.161 & 0.991 & 1.266 & 1.108 \\
  \bottomrule
  \end{tabular}
  \caption{RCG and OSR at each temperature.
  At $t = 1.5$, 9.7\% (WVS) and 9.9\% (ANES) of the answers of Qwen3.5-9B cannot be parsed and are treated as missing; for the other open-weight models, the rate is at most 1.1\%.}
  \label{table:ablation_temperature}
  \end{table}

\autoref{table:ablation_temperature} shows RCG and OSR at each temperature. 
For RCG, the largest range is 0.033 (Qwen3.5-9B on WVS), and RCG stays above Marginals at every temperature.
For OSR, the range on each dataset is 0.035 to 0.275 for gpt-5.4-nano, Qwen3.5-9B,
and gpt-oss-20b, but at most 0.025 for the other three models; the largest movement is in Qwen3.5-9B on WVS,
from 0.812 at temperature 0.1 to 1.087 at 1.5.
For gpt-5.4-nano, Qwen3.5-9B, and gpt-oss-20b, OSR on WVS crosses one as the temperature increases.
\citet{Wang2025-zs} report a finding consistent with this pattern for free-response text:
raising the temperature does not bring the diversity of GPT-4 responses up to
the level of human participants.

\section{Reasoning Effort Ablation}\label{sec:appendix:ret_effort}
  \begin{table}[t]
  \centering
  \small
  \setlength{\tabcolsep}{4pt}
  \begin{tabular}{lrrrr}
  \toprule
   & \multicolumn{2}{c}{WVS} & \multicolumn{2}{c}{ANES} \\
  \cmidrule(lr){2-3} \cmidrule(lr){4-5}
   & RCG & OSR & RCG & OSR \\
  \midrule
  \multicolumn{5}{l}{gpt-5.4-nano} \\
  \quad \texttt{none} (default) & 1.270 & 0.990 & 1.390 & 1.189 \\
  \quad \texttt{low} & 1.225 & 0.945 & 1.280 & 1.059 \\
  \quad \texttt{medium} & 1.214 & 0.919 & 1.266 & 1.025 \\
  \quad \texttt{high} & 1.213 & 0.917 & 1.253 & 1.020 \\
  \midrule
  \multicolumn{5}{l}{gpt-oss-20b} \\
  \quad \texttt{low} (default) & 1.268 & 1.003 & 1.338 & 1.163 \\
  \quad \texttt{medium} & 1.237 & 0.932 & 1.326 & 1.096 \\
  \quad \texttt{high} & 1.246 & 0.902 & 1.338 & 1.060 \\
  \midrule
  Marginals & 1.161 & 0.991 & 1.266 & 1.108 \\
  Copula & 1.040 & 0.947 & 1.086 & 0.991 \\
  LCA & 1.111 & 0.964 & 1.149 & 0.964 \\
  \bottomrule
  \end{tabular}
  \caption{RCG and OSR at each reasoning effort level.}
  \label{table:ablation_effort}
  \end{table}

Raising the reasoning effort shifts both models toward contraction,
while RCG falls most when reasoning is turned on.
We vary the reasoning effort, which controls how much the model reasons before answering,
over \texttt{low}, \texttt{medium}, and \texttt{high} for gpt-5.4-nano (default none) and
over \texttt{medium} and \texttt{high} for gpt-oss-20b (default \texttt{low}).
\autoref{table:ablation_effort} shows the results.\footnote{Unparsable
answers are treated as missing;
they make up at most 1.6\% of the answers in each setting and dataset.
Removing from all respondents the five questions (one in WVS and four in ANES) on which they concentrate changes RCG and OSR by at most 0.008.}
Turning reasoning on for gpt-5.4-nano (\texttt{none} to \texttt{low}) lowers RCG by 0.045 on WVS
and 0.110 on ANES, whereas raising the effort beyond \texttt{low} changes RCG by
at most 0.031 in either model.
OSR drops in all four model--dataset pairs, by 0.073 to 0.169 from the default to \texttt{high}.
RCG stays above Marginals except for gpt-5.4-nano on ANES,
which equals Marginals at \texttt{medium} (1.266) and falls below it at \texttt{high} (1.253).

\section{Model Tier Ablation}\label{sec:appendix:ret_tier}
  \makeatletter\setlength{\@fptop}{0pt}\setlength{\@fpsep}{8pt plus 2fil}\setlength{\@fpbot}{0pt plus 1fil}\makeatother
  \begin{table}[t]
    \centering
    \small
    \setlength{\tabcolsep}{3pt}
    \begin{tabular}{@{}lrrrr@{}}
    \toprule
     & \multicolumn{2}{c}{WVS} & \multicolumn{2}{c}{ANES} \\
    \cmidrule(lr){2-3} \cmidrule(lr){4-5}
     & RCG & OSR & RCG & OSR \\
    \midrule
    \multicolumn{5}{@{}l}{OpenAI} \\
    \quad gpt-5.4-nano (default) & 1.270 & 0.990 & 1.390 & 1.189 \\
    \quad gpt-5.4 & 1.199 & 0.692 & 1.262 & 0.838 \\
    \midrule
    \multicolumn{5}{@{}l}{Google} \\
    \quad gemini-flash-lite (default) & 1.218 & 0.829 & 1.356 & 0.878 \\
    \quad gemini-pro & 1.238 & 0.692 & 1.250 & 0.763 \\
    \midrule
    Marginals & 1.161 & 0.991 & 1.266 & 1.108 \\
    Copula & 1.040 & 0.947 & 1.086 & 0.991 \\
    LCA & 1.111 & 0.964 & 1.149 & 0.964 \\
    \bottomrule
    \end{tabular}
    \caption{RCG and OSR of the default and higher-tier models.}
    \label{table:ablation_tier}
    \end{table}

Raising the model tier mostly lowers RCG but shifts each model toward contraction.
We use gpt-5.4 from OpenAI, a higher-tier model of the same generation as gpt-5.4-nano,
and gemini-3.1-pro-preview (gemini-pro), a frontier model from Google.
In our runs, gpt-5.4 costs approximately 12 times as much per API call as gpt-5.4-nano, and gemini-pro costs approximately 33 times as much as gemini-flash-lite.

\autoref{table:ablation_tier} indicates that raising the tier lowers RCG in three of the four model--dataset pairs, by 0.071 to 0.128, and in the
remaining one raises it by 0.020 (gemini-pro on WVS); on ANES, both higher-tier models fall below Marginals.
OSR also drops in all four pairs and crosses one only for gpt-5.4 on ANES (from 1.189 to 0.838).

\end{document}